\documentclass[a4paper]{article}

\usepackage{geometry}
\usepackage{subfigure}
\usepackage{epsfig}
\usepackage{url}
\usepackage{amsfonts}
\usepackage{tipa}
\usepackage{times}
\usepackage{mathrsfs}
\usepackage{amssymb}
\usepackage{amsthm}
\usepackage{amsmath}

\newtheorem{definition}{Definition}[section]
\newtheorem{lemma}{Lemma}[section]
\newtheorem{theorem}{Theorem}[section]
\newtheorem{corollary}{Corollary}[section]

\newtheorem{proposition}{Proposition}[section]
\newtheorem{fact}{Fact}[section]

\newcommand{\Var}{\textrm{Var}}
\newcommand{\vol}{\textrm{vol}}

\newcommand{\poly}{\textrm{poly}}

\begin{document}

\title{Structural Information Learning Machinery: Learning from Observing, Associating, Optimizing, Decoding, and Abstracting\footnote{The author was partially supported by NSFC grant No. 61932002 and No.
61772503.}}

\author{Angsheng Li\\
State Key Laboratory of Software Development Environment\\
 School of Computer Science,
 Beihang University\\
 Beijing, 100083, P. R. China.}

\maketitle

\begin{abstract}

Both computation and information are the keys to understanding learning and intelligence. However, the studies of computation and information had been largely separated in the academic communities, for which 
a fundamental question in the theoretical underpinnings of information science and computer science is to measure the information that is embedded in a physical system \cite{B2003}. 
The author and his co-author \cite{LP2016a} introduced the notion of {\it encoding tree} as a lossless encoding of a graph and the metric of {\it structural entropy} of graphs. The structural entropy of a graph is the {\it intrinsic information hidden in the graph} that cannot be decoded by any encoding tree or any lossless encoding of the graph. The structural information is defined as a concept of the merging of computation and information.
In the present paper, we propose the model of {\it structural information learning machines} (SiLeM for short), leading to a mathematical definition of learning by merging the theories of computation and information. 
Our model shows that the essence of learning is {\it to gain information}, that to gain information is {\it to eliminate uncertainty} embedded in a data space, and that to eliminate uncertainty of a data space can be reduced to an optimization problem, that is, an {\it information optimization problem}, which can be realized by a general {\it encoding tree method}. The principle and criterion of the structural information learning machines are maximization of {\it decoding information} from the data points observed together with the relationships among the data points, and semantical {\it interpretation} of syntactical {\it essential structure}, respectively. 
A SiLeM machine learns the laws or rules of nature. It observes the data points of real world, builds the {\it connections} among the observed data and constructs a {\it data space}, for which the principle is to choose the way of connections of data points so that the {\it decoding information} of the data space is maximized, finds the {\it encoding tree} of the data space that minimizes the dynamical uncertainty of the data space, in which the encoding tree is hence referred to as a {\it decoder}, due to the fact that it has already eliminated the maximum amount of uncertainty embedded in the data space, interprets the {\it semantics} of the decoder, an encoding tree, to form a {\it knowledge tree}, extracts the {\it remarkable common features} for both semantical and syntactical features of the modules decoded by a decoder to construct {\it trees of abstractions}, providing the foundations for {\it intuitive reasoning} in the learning when new data are observed.
Our SiLeM machines learn from observing, associating, encoding, optimizing, decoding, interpreting, abstracting and intuitive reasoning to realize the maximum gain of information, without any hand-made choice of parameter.

\end{abstract}

\section{\bf Introduction}\label{sec:int}

Turing machines \cite{T1936} capture the mathematical essence of the concept of ``computation", give not only a mathematical definition of the concept computation, but also provides a model to build ``computers". In the 20th century, it had been proved that computers are useful, for which the mission of computer science was to develop efficient algorithms and computing devices.  In the 21st century, computers have been becoming very useful everywhere. The mission of computers has become ``information processing" in the real world. However, there is no a mathematical theory that supports the mission of ``information processing".

At the beginning of artificial intelligence in 1956, one point of view was to regard ``artificial intelligence" as ``complex information processing". Again, there was no mathematical understanding of complex information processing.

In the past more than 70 years, Shannon's information theory is the main principle for us to understand the concept of ``information". However, Shannon's theory fails to support the current ``information processing", especially ``complex information processing".

Shannon's \cite{S1948} metric measures the uncertainty of a probabilistic distribution or a random variable from the probability distribution as

\begin{equation}\label{eqn:Shannon}
H(p_1,\cdots,p_n)=-\sum\limits_{i=1}^np_i\log_2p_i.
\end{equation}

This metric and the associated concept of noise, have provided rich sources for both information theory and technology.
In particular, Shannon's theory solved two fundamental questions in communication theory: What is the ultimate data compression, and what is the ultimate transmission rate of communication. For this reason, some people consider information theory to be a subfield of communication theory. We remark that it is much more. Indeed, information theory plays an important role in many areas, such as statistical physics, computer science, statistical inference, probability and statistics. 

Shannon's metric measures the quantity of uncertainty embedded in a random variable or a probability distribution. We note that either a random variable or a probability distribution is a function. Functions are classical objects in mathematics, representing the correspondence from every individual of a set to an element of the same or another set. However, in the real world, we often have to deal with systems consisting of many bodies and the relationships among the many bodies, referred to as physical systems. To represent such systems, graphs are the general mathematical model. Therefore, graphs are natural extensions of functions, and are general models of representations of real world objects. Shannon's theory indicates that, there is a quantity of uncertainty in random variables. We know that a random variable is in fact a function, and that a function is a special type of graph. Due to the fact that there are uncertainty in random variables and that graphs are natural extensions of functions,  there must exist uncertainty in graphs. However, Shannon's metric fails to measure the quantity of uncertainty embedded in a physical system such as a graph. In 2003, Brooks \cite{B2003} commented that: `` We have no theory, however, that gives us a metric for the information embedded in structure, especially physical structure". In addition, Shannon \cite{S1953} himself realized that his metric of information fails to support the analysis of communication networks to answer questions such as the characterization of optimum communication networks.

As a matter of fact, graph compressing and structure decoding are fundamental questions in structured noisy data analysis. However,
literature on graphical structure compression is scare. Turn \cite{T1984} introduced the problem of succinct representation of general unlabelled graphs. Naor \cite{N1990} provided such a representation when all unlabelled graphs are equally probable. Adler and Mitzenmacher \cite{AM2001} implemented some heuristic experiments for real-world graph compression. Sun, Bolt and Ben-Avraham \cite{SBB2008} proposed an idea similarly to that in \cite{AM2001} to compress sparse graphs. Peshkin \cite{P2007} proposed an algorithm for a graphical extension of the one-dimensional SEQUITUR compression method. Choi and Szpankowski \cite{CS2009} proposed an algorithm for finding the Shannon entropy of a graph generated from the ER model.

To understand the information embedded in a graph, we will need to encode the graph. How to encode a graph? In graph theory, there are parameters related to three types of graph encoding. Each model of these encodings involves assigning vectors to vertices, and the parameter is the minimum length of vectors that suffice. We study the maximum of this parameter over $n$-vertex graphs. The parameters are intersection number, product dimension, and squashed-cube dimension. 
Erd\"os, Goodman and P\'osa \cite{EGP1966} proposed the definition of intersection number and studied the notion. An intersection representation of length $t$ assigns each vertex a $0,1$-vector of length $t$ such that $u$ and $v$ have an edge if and only if their vectors have a $1$ in a common position. Equivalently, it assigns each $x\in V$ a set $S_x\subseteq [t]=\{1,2,\cdots,t\}$ such that for any $u,v$, there is an edge $(u,v)$ if and only if $S_u\cap S_v\not=\emptyset$.  The second parameter is the product dimension. A product representation of length $t$ assigns the vertices distinct vectors of length $t$ so that there is an edge $(u,v)$ if and only if their vectors differ in every position. The product dimension of a graph $G$ is the minimum length of such a representation of $G$. Lova\'sz, Nesetril and Pultr \cite{LNP1980} characterized the $n$-vertex graphs with product dimension $n-1$. The third encoding is to assign vectors to vertices such that distance between vertices in the graph is the number of positions where their vectors differ. Each of these encoding assigns vectors to vertices to preserve certain properties of the graphs. The key point of the encodings is to use the mathematical operations of vectors to 
recover the properties of the graphs. Clearly, operations over vectors are easy and more efficient. Unfortunately, these encodings distort the graphs, although each of them may preserve some specific properties of graphs.

To establish a theory of the information embedded in graphs, we will need a lossless encoding of graphs. Is there a lossless encoding of graphs? 
The author and his co-author \cite{LP2016a} introduced the concept of encoding tree of a graph as a lossless encoding of graphs, and defined the structural entropy of a graph to be the minimum amount of information required to determine the codeword of the vertex in an encoding tree for the vertex that is accessible from random walk with stationary distribution in the graph, under the condition that the codeword of the starting vertex of the random walk is known. The structural entropy of a graph is hence the {\it intrinsic information} embedded in the graph that cannot be decoded by any encoding tree or any lossless encoding of the graph. Measuring the structural entropy of a graph involves finding an encoding tree of the graph under which the information required to determine the codeword of vertices accessible from random walk in $G$ when the codeword of the starting vertex of the random walk is known is minimized. The quantification of the {\it structural entropy} of a graph defined in this way is {\it the intrinsic information hidden in the graph $G$} that cannot be decoded by any encoding tree or any lossless encoding of the graph.
The encoding tree found in this way, that is, minimizing the information hidden in a graph $G$, in the measuring of structural entropy of the graph $G$ hence determines and decodes a structure of $G$ by using which the uncertainty still left or hidden in $G$ has been minimized. We thus call such an encoding tree $T$ of $G$ a {\it decoder of $G$}. The decoder, $T$ say, of graph $G$ is hence an encoding tree of $G$. Since $T$ determines an encoding under which the uncertainty left or still hidden in $G$ is minimized, the syntactic structure $T$ of $G$ certainly supports a semantical or functional modules of $G$. More precisely, a decoder $T$ supports a semantical interpretation of the system $G$. Due to the fact that the decoder is an encoding tree, the semantical interpretation supported by the decoder is hence called a {\it a knowledge tree} of $G$. This provides a general principle to acquire knowledge from observed dataset.
This strategic goal of structural information theory has been successfully verified in real world applications. 
In \cite{LP2016a, LYP2016}, we established a systematical method based on the structural entropy minimization principle, without any hand-made parameter choices, to identify the types and subtypes of tumors. The types and subtypes found by our algorithms of structural entropy minimization are highly consistent with the clinical datasets. In \cite{L+2018}, we developed a method, referred to as deDoC, based on the principle of structural entropy minimization, to find the two- and three-dimensional DNA folded structures. The deDoC was proved the first principle-based, systematical, massive method for us to precisely identify the {\it topologically associating domains} (TAD) from Hi-C data. Remarkably, deDoC finds TAD-like structures from $10$ single cells. This opens a window for us to study single cell biology, which is crucial for potential breakthroughs in both biology and medical sciences. In network theory and network security, 
the concept of structural entropy \cite{LP2016a} has been extended to measure the security of networks \cite{LLPZ2015, LP2016, LZP2016, Liu+2019}.

Structural information, as a result of the merging of the concepts of computation and information, has a rich theory, referred to \cite{LP2016a}. More importantly, the concept of structural information provides a key to mathematically understanding the principle of data analysis, the principle of learning, and even the principle of intelligence. The reasons are as follows: Computing is, of course, an ingredient of learning and intelligence. ``Information", if well-defined, must be the the foundation of intelligence. Mathematically speaking, entropy is the quantity of uncertainty, and information is the amount of uncertainty that has been eliminated. Therefore, both computation and information are the keys for us to understand the mathematical essence of ``intelligence". However, in the past more than 70 years, the studies of computational theory and information theory are largely separated. Consequently, we have no idea on how the two keys of computation and information open the window for us to capture the concept of intelligence. The structural information theory, as a theory of the merging of computation and information, opens such a window.

In the present paper, we propose the model of {\it structural information learning machinery}, written SiLeM. Our structural information learning machines assume that {\it observing} is the basis of learning, that laws or rules are embedded in a noisy system of observed dataset in which each element usually consists of a {\it syntax}, a {\it semantics} and {\it noises}. Our machines learn the laws or rules of real world by {\it observing} the datasets, by using the {\it principle of maximization of information gain} to {\it connect} the datasets and to build a {\it data space}, by a {\it general encoding tree method} to {\it decode} (using the {\it structural entropy minimization principle}) the structural information of a data space to find the {\it decoder} or {\it essential structure} of the data space, by using the semantics of data points to interpret the essential structure or decoder of a data space to build a {\it knowledge tree} of the data space and to {\it unify both syntax and semantics} of the data space, solving the problem of {\it interpretability} of learning, by using {\it remarkable common features} of functional modules to {\it abstract} the decoder or knowledge tree to establish a {\it tree of abstractions}, by using the tree of abstractions in the encoding and optimizing when new data points are observed to realize both {\it intuitive reasoning} and {\it logical reasoning} simultaneously. Our learning model shows that learning from observing is possible, that laws or rules exist in the relationships among the data points observed, that the combination of both syntax and semantics is the principle for solving the interpretability problem of learning, that simultaneously realizing both logical reasoning and intuitive reasoning is possible in learning, that the mathematical essence of learning is to gain information, and maximization of information gain is the principle for learning algorithms that are completely free of hand-made choice of parameters. Our model shows that computing is part of learning. However, computing and learning are mathematically different concepts.

We organize the paper as follows. 
In Section \ref{sec:challenges}, we introduce the challenges of the current machine learning. 
In Section \ref{sec:overview}, we introduce the overview of our structural information learning machines.
In Section \ref{sec:structuralentropy}, we introduce the concepts of structural entropy of graphs \cite{LP2016a}, and prove some new results about the equivalent definitions of the structural entropy.
In Section \ref{sec:extension-Shannon}, we show that the structural entropy is a natural extension of the Shannon entropy from unstructured probability distribution to structured systems, and prove a general lower bound of structural entropy which will be useful for us to understand the present structural information learning machines.
In Section \ref{sec:structuredocoding}, we introduce the concepts of compressing information and decoding information of graphs, establish a graph compressing/decoding principle, and establish an upper bound of the compressing information of graphs.
In Section \ref{sec:decoder}, we introduce the concepts of decoder, knowledge tree and rule abstraction, and establish a structural information principle for clustering and for unsupervised learning.
In Section \ref{sec:observing-connecting-associating}, we establish the structural information principle for connecting and associating data when new dataset is observed.
 In Section \ref{sec:logical-intuitive-reasoning}, we introduce the definition and algorithms for both logical and intuitive reasonings of our learning model. In Section \ref{sec:SiLeM}, we introduce the system of structural information learning machinery. In Section \ref{sec:encoding-tree-method}, we introduce the encoding tree method as a general method for designing algorithms of the structural information learning machinery. In Section \ref{sec:limits-of-SiLeM}, we introduce the limitations of our structural information learning machinery. In Section \ref{sec:con}, we summarize the contributions of the structural information learning machinery, and introduce some potential breakthroughs of the machinery.

\section{The Challenges of Learning and Intelligence}\label{sec:challenges}

Mathematical understanding of learning has become a grand challenge in the foundations of both current and future artificial intelligence.

Statistical learning is a branch with successful theory. Overall, statistical learning is a learning of the approach of the combination of computation and statistics. Statistics provides the principle for statistical results. Computation has two fundamental characters: one is locality, another is structural property. Consider a procedure of aTuring machine, at any time step in the procedure, the machine focuses only on a few states, symbols, and cells on the working tape. This is the character of locality. In addition, due to the fact that algorithms are always closely related to data structure (since, otherwise, the objects are statistical, instead of computational), computation has its second character, structural property. Of course, the approach of the combination of computation and statistics is very successful in both theory and applications. However, nevertheless, statistical learning does not really tell us what is exactly the mathematical essence of learning.

For deep learning, as commented in \cite{LBH2015}:``Unsupervised learning had a catalytic effect in reviving interest in deep learning, but has since been overshadowed by the successes of purely supervised learning. $\cdots$ Human and animal learning is largely unsupervised: we discover the structure of the world by observing it, not by being told the name of every object." 

Both supervised and unsupervised learning have been very successful in many real world applications. However, we have to recognize that we still do not know what is exactly the mathematical essence of learning and intelligence.

In particular, are there machines that learn by observing the real world similar to human learning?
Is there a mathematical definition of learning, similar to the mathematical definition of computing given by Turing \cite{T1936}? What are the fundamental differences between learning and computing? Are intelligences really just function approximations? What are the fundamental differences between learning and intelligence, between learning and computing, between learning and information, and between information and intelligence?

The current learning theory is built based on function approximations.
Functions are essentially mathematical objects, which are defined by mathematical systems. Due to this fact, mathematical functions usually have only syntax, do not have semantics, and do not have noises. If learning or intelligence were just function approximation, then we would learn only mathematics. However, human learns mathematics, physics, chemistry, biology and so on. As a matter of fact, human learns from the real world and learns the laws of the nature. Human learns the laws of the nature principally based on observing, connecting data, associating, computing, interpreting and reasoning, including both logical reasoning and intuitive reasoning. When human beings learn, people use eyes to see, use brain to reason, use hand to calculate, and use mouth to speak aloud etc. When human beings learn, intuitive reasoning is equally important to logical reasoning, if it is not more important. Logical reasoning is actually a type of computation. Thinking of a Turing machine, we note that computation is locally performed, in the sense that, during the procedure of a computation, the machine focuses only on the head of the machine, which points to a cell and moves either to the left or to the right one more cell in a working tape. Computation is certainly a factor of learning. However, human learning includes both computation and intuitive reasoning, where intuitive reasoning is a reasoning by using the laws and knowledges one has already learnt. This argument shows, it is not the case that learning is another type of computing. Intuitively speaking, computation is a mathematical concept, dealing with only mathematical objects, that is, computable functions or computing devices, but learning is a concept dealing with real world objects.

What are the differences between mathematical objects and real world objects?
Mathematical objects largely consist of only syntax. However, real world objects certainly consist of syntax, semantics and noises. Human beings learn different objects, which have different semantics. For instance, the subjects such as mathematics, physics and chemistry etc are different due to the fact that they have different semantics. However, the mathematical essence of the learning of these different subjects could be still the same. If so, this would lead to a mathematical definition of the concept of ``learning".  What is the mathematical definition of ``learning"?

Computer science has been experiencing a big change from the 20th century to the 21st century. In the 20th century, computer science is largely proven to be useful. However, in the 21st century, computer has been proven to be useful everywhere. This changes the universe of computer science from ``mathematics and computing devices" to the ``real world". Computing the real world is roughly stated as ``information processing" from the datasets observed from real world. 

However, there is no a mathematical theory that supports the mission of information processing.
To understand the concept of ``information processing", we look at the information theory. Shannon's information theory perfectly supports the point to point communication. However, it fails to support the analysis of communication networks, as noticed by Shannon himself \cite{S1953}. Apparently, Shannon's information theory fails to support the current information processing practice of computer science.  
Shannon's metric defines entropy as the amount of uncertainty of a random variable, and mutual information as the amount of uncertainty of a random variable, $X$ say, that is eliminated by knowing another random variable, $Y$ say. This means that ``information" is the amount of uncertainty that has been eliminated. However, Shannon's theory deals with only random variables or probability distributions. In addition, although Shannon defined the concept of  ``information" as the amount of uncertainty that has been eliminated, Shannon did not say anything about: Where does information exist? How do we generate information? How do we decode information?

In the 20th century, the studies of computation and information were largely separated, developed in computer science and communication engineering, respectively. The argument above indicates that there is a need of study of the combination of computation and information. In fact, the current society is basically supported by several massive systems each of which consists of a large number of computing devices and communication devices, which calls for a supporting theory in the intersection of computational theory and information theory. Brooks 2003 \cite{B2003} explicitly proposed the question of ``quantification of structural information". In the same paper, Brooks commented that ``this missing metric to be the most fundamental gap in the theoretical underpinnings of information science and of computer science".

The author and his co-author \cite{LP2016a} introduced the notion of encoding tree of graphs as a lossless encoding of graphs, defined the first metric of information that is embedded in a graph, and established the fundamental theory of structural information. 
The structural entropy of a graph is defined as the {\it intrinsic information hidden in the graph} that cannot be decoded by any encoding tree or any lossless encoding of the graph. 
The structural information theory is a new theory, representing the merging of the concepts of computation and information. It allows us to combine the fundamental ideas from both coding theory and optimization theory to develop new theories. More importantly, the new theory points to some fundamental problems in the current new phenomena such as massive data analysis, information theoretical understanding of learning and intelligence.

It is not hard to see that both computation and information, and the combination of the two concepts are the keys to better understand the mathematical essence of learning and intelligence. The separation of the studies of computation and information in the past more than 70 years has hindered the theoretical progress on both learning and intelligence. Structural information theory provides a new chance.

\section{Overview of Structural Information Learning Machines}\label{sec:overview}

In the present paper, we will build a new learning model, namely, the structural information learning machinery. Our model is built based on our structural information theory \cite{LP2016a}. Our learning model is a mathematical model that exactly reflects the merging of computation and information.
Our theory of information theoretical definition of learning here provides new approaches to potential breakthroughs in a wide range of machine learning and artificial intelligence.

The machines of model SiLeM learn the laws or rules of nature by observing the data of the real world. The mathematical essences of SiLeM  are: (1) the essence of learning is {\it to gain information}, (2) to gain information is {\it to eliminate uncertainty}, and (3) according to the principle of structural information theory, to eliminate uncertainty of a data space can be reduced to an optimization problem, that is, an {\it information optimization problem}, by a general {\it encoding tree method}. A SiLeM machine observes the data points of real world, builds the {\it connections} among the observed data, constructs a {\it data space} (for which the principle is to choose the way of connections of data points so that {\it the information gain} from the data space is maximized), finds the {\it encoding tree} of the data space that minimizes the uncertainty of the data space, in which the encoding tree is also referred to as a {\it decoder} due to the fact that it eliminates the maximum amount of uncertainty embedded in the data space, interprets the {\it semantics} of the decoder, an encoding tree, to form a {\it knowledge tree}, extracts the {\it laws or rules} of both the decoder and the knowledge tree. The decoder and knowledge tree of a graph determines a {\it tree of abstractions} which defines the concept of hierarchical abstracting and provides the foundation for {\it intuitive reasoning} in learning. When new dataset are observed, a SiLeM machine updates the decoder, i.e., an encoding tree, by using the tree of abstractions extracted from the decoder and knowledge tree found from the previous data space.

Our SiLeM machines assume that a data point representing a real world object usually consists of a syntax, a semantics and a noise, that the laws or rules of the real world objects are embedded in a noisy data space, that the functional semantics of the data space must be supported by an {\it essential structure} of the data space, and that the essential structure of a data space is the encoding tree of the data space that minimizes the uncertainty left in the data space, or maximumly eliminates the uncertainty embedded in the data space.

A SiLeM machine realizes the mechanism of {\it associating} through linking data to existing data apace and to {\it established knowledge and laws} in the tree of abstractions, a procedure highly similar to human learning, realizes the unification of syntactic and semantical interpretations, solving the problem of interpretability of learning, and more importantly, simultaneously realizes both {\it logical reasoning} (that is, the local reasoning of computation and optimization) and {\it intuitive reasoning} (that is, the global reasoning by using laws and knowledge learnt previously).

The mathematical principle behind the procedure of SiLeM machines is to realize the {\it maximum gain of information}, by linking data points to existing dataset in a way such that the constructed data space contains the maximum amount of {\it decodable information}, the amount of uncertainty that can be eliminated by an encoding tree, or by a lossless encoder, instead of the information hidden in the data space eventually and forever, and by
maximumly eliminating the uncertainty embedded in the data space that is realized by using an {\it information optimization}, which is efficiently achievable by an {\it encoding tree method}.
Our structural information learning machines explore that the essence of learning is {\it to gain information} from the datasets observed, together with the relationships among the data points, that to gain information is {\it to eliminate uncertainty}, and more importantly, to eliminate uncertainty can be reduced to an {\it information optimization problem}, which can be efficiently realized by a general encoding tree method.

%\section{Entropy vs Information}

\section{Structural Entropy of Graphs}\label{sec:structuralentropy}

To develop our information theoretical model of learning, we recall the notion of structural entropy of graphs \cite{LP2016a}.

To define the structural entropy of a graph, we need to encode a graph. It has been a long-standing open question to build a lossless encoding of a graph. In graph theory, there are several encodings of graphs, each of which encodes a graph by assigning high-dimensional vectors to the vertices of the graph. In doing so, operations in graphs can be reduced to operations in vector spaces. However, such encodings usually distort the structure of the graph, due to the fact that the operations of vectors do not exactly reflect the operations in the corresponding graphs. 

Our idea is to encode a graph by a tree. Trees are the simplest graphs in some sense. Why do we use trees to encode a graph? There is no mathematical proof for this. However, we have reasons as follows.

 Suppose that $G$ is a graph observed in the real world. Then $G$ represents the syntactical system of many objects together with the relationships among the objects. In addition, there is a semantics that is associated with, but outside of the system $G$. The semantics of $G$ is the knowledge of system $G$. The knowledge of $G$ is typically a structure of the form of functional modules of system $G$. In this case, the knowledge of system $G$ is a structure of functional modules associated with $G$. What is the structure of the knowledge, or functional modules or semantics of a system $G$?

To answer the questions, we propose the following {\bf hypothesis}:

\begin{enumerate}
\item [(1)] The semantics of a system,  representing the functional modules or roles of the system, has a {\it hierarchical structure}. 

This hypothesis reflects the nature of human understanding for a complex system consisting of many bodies together with the relationships among the many bodies. It is true that given a complex system consisting of a huge number of real world objects together with their relationships, people can only understand it by identifying the functional modules of the complex system by a tree-like structure or by a hierarchical structure. The hierarchical structure of functional modules gives us a hierarchical or tree-like abstractions of the system. We understand a complex system by a high-level abstractions. This means that humans understand the functional modules of a complex system by a hierarchical structure, or by a tree-like structure.

In addition, we assume that human organizes knowledges as a tree structure, and hence that human knowledges have a tree structure.

\item [(2)] The semantics of a system has a supporting syntax, referred to as {\it essential structure} of the system.

This means that semantics certainly has a supporting syntax structure.

\item [(3)] According to (1) and (2) above, the essential structure (syntax) of a system $G$ has a hierarchical structure.

Because the semantics of a system has a tree structure, the supporting syntax must have a tree structure. This supporting tree structure is called the essential structure of the system.

\end{enumerate}

The hierarchical hypothesis implies that the essential structure, that is, the supporting syntax of a complex system $G$ is a tree. This suggests us to encode a complex system by trees. 

Furthermore, we notice that:

\begin{enumerate}

\item [(i)] From the point of view of human understanding of knowledges, human understands complex systems by a functional modules of high-level abstractions.

\item [(ii)] From the point of view of computer science, trees are efficient data structures, representing systems of many objects, and simultaneously allowing highly efficient algorithms.

\item [(iii)] From the point of view of information theory, trees provide the fundamental properties needed for encoding, see the {\it Encoding Tree Lemma} in Lemma \ref{lem:codeword} below.

\end{enumerate}

Nevertheless, in \cite{LP2016a}, we encoded graphs by trees. Specifically, we used the priority tree defined below to encode a complex system.

\subsection{Priority tree}

\begin{definition} (Priority tree) A priority tree is a rooted tree $T$ with the following properties:

\begin{enumerate}
\item [(i)] The root node is the empty string, written $\lambda$.

A node in $T$ is expressed by the string of the labels of the edges from the root to the node.
We also use $T$ to denote the set of the strings of the nodes in $T$.

\item [(ii)] Every non-leaf node $\alpha$ in $T$ has $k\geq 2$ children for some natural number $k$ (depending on $\alpha$) for which the edges from $\alpha$ to its children, or referred to as immediate successors, are labelled by:

$$0<_{\rm L}1<_{\rm L}\cdots <_{\rm L}k-1,$$
\noindent where $x<_{\rm L}y$ denotes that $x$ is to the left of $y$.

({\it Remark: (i)
Unlike Huffman codes \cite{H1952}, we use an alphabet of the form $\Sigma=\{0,1,\cdots, k-1\}$ for some natural number $k$, for each non-leaf tree node $\alpha$. In the Huffman codes, we always use the alphabet $\Sigma=\{0,1\}$. For measuring the number of bits used in an encoding, we usually use binary trees. However, for our purpose, there is no reason to prevent us from using general alphabet $\Sigma=\{0,1,\cdots,k\}$. We are interested in trees in general, instead of binary trees only.

(ii) Different non-leaf nodes in $T$ may have different numbers of immediate successors (or simply, called children), i.e., different $\alpha$'s may have different $k$'s.)}

\item [(iii)] Every tree node $\alpha$ is hence a string of numbers from $0$ to some natural number, $K$ say.

\end{enumerate}

\end{definition}

For two tree nodes $\alpha, \beta$, if $\alpha$ is an initial segment of $\beta$ as string, then we write $\alpha\subseteq\beta$. If $\alpha\subseteq\beta$ and $\alpha\not=\beta$, we write $\alpha\subset\beta$.

(Remark: The motivation of the use of priority tree above is to leave a room for us to develop an encoding tree method, in which order plays a role.)

\subsection{Encoding tree of a graph}

\begin{definition} (Encoding tree of a graph) \label{def:encoding-tree} Let $G=(V,E)$ be a graph. An encoding tree of $G$ is a priority tree $T$ such that for every tree node $\alpha\in T$, there is an associated non-empty subset $T_{\alpha}$ of the vertices $V$ satisfying the following properties:

\begin{enumerate}
\item [(i)] The root node $\lambda$ is associated with the whole set $V$ of vertices of $G$, that is, $T_{\lambda}=V$.
\item [(ii)] For every node $\alpha\in T$, if $\beta_1,\beta_2,\cdots,\beta_k$ are all the children of $\alpha$, then $\{T_{\beta_1},\cdots, T_{\beta_k}\}$ is a partition of $T_{\alpha}$.
\item [(iii)] For every leaf node $\gamma\in T$, $T_{\gamma}$ is a singleton.

\end{enumerate}

\end{definition}

\begin{definition} (Codeword) Let $G=(V,E)$ be a graph, and $T$ be an encoding tree of $G$.

\begin{enumerate}
\item [(i)] For every node $\alpha\in T$, we call $\alpha$ the codeword of set $T_{\alpha}$, and $T_{\alpha}$ the marker of $\alpha$.
\item [(ii)] For a leaf node $\gamma\in T$, if $T_{\gamma}=\{v\}$ for some vertex $v\in V$, then we say that $\gamma$ is the codeword of $v$, and $v$ is the marker of $\gamma$.
\end{enumerate}

\end{definition}

\begin{lemma} \label{lem:codeword} Let $G=(V,E)$ be a graph and $T$ be an encoding tree of $G$. Then:

\begin{enumerate}
\item [(1)] For every leaf node $\gamma\in T$, there is a unique vertex $v$ such that $\gamma$ is the codeword of $v$ and $v$ is the marker of $\gamma$.
\item [(2)] For every vertex $v\in V$, there is a unique leaf node $\gamma\in T$ such that $v$ is the marker of $\gamma$ and $\gamma$ is the codeword of $v$.

\end{enumerate}

\end{lemma}
\begin{proof} By the definition of encoding tree in Definition \ref{def:encoding-tree}.
\end{proof}

By Lemma \ref{lem:codeword}, the set of all the leaves in $T$ is the set of codewords of the vertices $V$. Clearly, we have that an encoding tree $T$ of $G$ is a lossless encoding of $G$. 

More importantly, an encoding tree $T$ of a graph $G$ satisfies the following:

\begin{lemma}\label{lem:coding tree property} (Encoding tree lemma) Given a graph $G$ and an encoding tree $T$ of $G$, the following properties hold:

\begin{enumerate}
\item [(1)] For every node $\alpha\in T$, the marker $T_{\alpha}$ of $\alpha$ is explicitly determined. This means that if we know $\alpha$, then we have already known the marker $T_{\alpha}$, i.e., there is no uncertainty in $T_{\alpha}$ once we know the codeword $\alpha\in T$.

    \item [(2)] For every vertex $x$ in $V$, suppose that we know the codeword of $x$ to be $\alpha$, then we simultaneously know the marker $T_{\beta}$ for all $\beta\subseteq\alpha$, that is, once we know $\alpha$, we know the path from the root $\lambda$ to $\alpha$, hence the markers associated on the path.

        \item [(3)] Let $x$ and $y$ be two vertices. Suppose that $\alpha$ and $\beta$ are the codewords of $x$ and $y$ in $T$, respectively. Let $\gamma$ be the longest $\delta$ such that both $\delta\subseteq\alpha$ and $\delta\subseteq\beta$ hold, that is, $\gamma$ is the longest common initial segment of $\alpha$ and $\beta$. Suppose that we know the codeword $\alpha$ of $x$, and know $y$. But we don't know the codeword $\beta$ of $y$. Then:

            \begin{enumerate}
            \item To {\it determine} (or define) the codeword $\beta$ of $y$ under the condition that we have already known the codeword $\alpha$ of $x$, we only need to determine the path from $\gamma$ to the unknown $\beta$ in the encoding tree $T$.
            \item To {\it describe} the codeword $\beta$ of $y$ in $T$, we must write down $\beta$, even if we know the codeword $\alpha$ of $x$.

            (This means that under the condition of knowing $\alpha$, to determine the codeword $\beta$ of $y$ is different from to describe the codeword $\beta$ of $y$, even if we have already known the codeword $\alpha$ of $x$.)

            \end{enumerate}

\end{enumerate}
\end{lemma}
\begin{proof}
For (1). By the definition of encoding tree $T$, whenever we define a tree node $\alpha\in T$, a subset $T_{\alpha}$ is explicitly defined.

For (2). Given a tree node $\alpha$, we have already known the path from the root node $\lambda$ to $\alpha$, because the path is unique. This implies that for every tree node $\delta$ on the path between $\lambda$ and $\alpha$, we have already known the associated marker $T_{\delta}$.

For (3). To find $\gamma$, we only need to find the find the longest $\delta$ from $\alpha$ along the path from $\alpha$ to the root node $\lambda$ such that $y\in T_{\delta}$. Then we know that $\beta$ must be some leaf node in a subtree of $T$ with root $\gamma$.

Therefore, the uncertainty to determine $\beta$ under the condition of knowing $\alpha$ and $y$ occurs only in a branch from $\gamma$ to some leaf node in $T$. 

\end{proof}

The advantage of the encoding tree is captured by the encoding tree properties in Lemma \ref{lem:coding tree property}. The key to our definition of structural entropy is to use the encoding tree properties in Lemma \ref{lem:coding tree property} to reduce the uncertainty for determining the codeword of the vertex that is accessible from random walk with stationary distribution in the graph under the condition that the codeword of the starting vertex of the random walk is known.

Lemma \ref{lem:coding tree property} (3) holds for any pair $(x,y)$ of vertices. However, in our definition of structural entropy, we will only use this property for the pairs $(x,y)$ with edges between the two endpoints. Because the structural entropy measures the information of random walks in a graph, that is, the dynamical information embedded in the graph.
For two vertices $x$ and $y$, if there is an edge from $x$ to $y$, then Lemma \ref{lem:coding tree property} (3)
indicates that to determine the codeword of the vertex, $y$ say, accessible from random walk when we know the codeword of the starting vertex, $x$ say, of the random walk is different from describing the codeword of the vertex $y$ accessible from random walk even if we have already known the codeword of the vertex $x$ at which the random walk starts. It is because of the difference between determining (defining) the codeword and describing (writing down) the codeword of the vertex accessible from random walk under the condition of the known codeword of the starting vertex of the random walk, our definition of structural entropy becomes different from the Shannon entropy. Therefore, Lemma \ref{lem:coding tree property} plays a crucial role in the definition of our metric of the structural entropy of a graph.

\subsection{Structural entropy of a graph given by an encoding tree}

Li and Pan \cite{LP2016a} introduced the notion of structural entropy of a graph.

\begin{definition} \label{def:structural entropy-T-O} (Structural entropy of a graph by an encoding tree, Li and Pan \cite{LP2016a}) Let $G=(V,E)$ be a graph, and $T$ be an encoding tree of $G$.
We define the structural entropy of $G$ with respect to encoding tree $T$ as follows:

\begin{equation}\label{eqn:H-T-O}
\mathcal{H}^{T}(G)=-\sum\limits_{\substack{\alpha\in T\\ \alpha\not=\lambda}}\frac{g_{\alpha}}{{\rm vol}(G)}\cdot \log_2\frac{{\rm vol}(\alpha)}{{\rm vol}(\alpha^{-})},
\end{equation}
where $g_{\alpha}=|E(\bar{T_{\alpha}}, T_{\alpha})|$, that is, the number of edges from the complement of $T_{\alpha}$, i.e., $\bar{T_{\alpha}}$, to $T_{\alpha}$, ${\rm vol}(G)$ is the volume of $G$, that is, the total degree of vertices in $G$, ${\rm vol}(\beta)$ is the volume of the vertices set $T_{\beta}$, and $\alpha^{-}$ is the parent node of $\alpha$ in $T$.

\end{definition}

%\subsection{The intuition of $\mathcal{H}^{T}(G)$}

%In this part, we establish a result 

To understand Equation (\ref{eqn:H-T-O}), we observe the following properties of the metric $\mathcal{H}^T(G)$:

\begin{enumerate}
\item [(1)] For every tree node $\alpha\in T$, $T_{\alpha}$ is the set of vertices associated with $\alpha$.
By Lemma \ref{lem:coding tree property} (1), once we know $\alpha$, we have already known the set $T_{\alpha}$.

\item [(2)] Suppose that we know tree node $\alpha$, then by Lemma \ref{lem:coding tree property} (2), we have already known all $\delta\subseteq\alpha$, meaning that $\delta$ is an initial segment of $\alpha$ as string.

\item [(3)] For each node $\alpha\in T$ with $\alpha\not=\lambda$, since $\alpha^{-}$ is the parent node of $\alpha$ in $T$, the probability that the vertex $v\in V$ from random walk with stationary distribution in $G$ is in $T_{\alpha}$ under the condition that $v\in T_{\alpha^{-}}$ is $\frac{{\rm vol}(\alpha)}{{\rm vol}(\alpha^{-})}$. Therefore the entropy (or uncertainty) of $v\in T_{\alpha}$ under the condition that $v\in T_{\alpha^{-}}$ is $-\log_2\frac{{\rm vol}(\alpha)}{{\rm vol}(\alpha^{-})}$.

\item [(4)] For every node $\alpha\in T$, $g_{\alpha}$ is the number of edges that random walk with stationary distribution arrives at $T_{\alpha}$ from vertices $\bar{T_{\alpha}}$, the vertices outside of $T_{\alpha}$. Therefore, the probability that a random walk with stationary distribution is from outside of $T_{\alpha}$ to vertex in $T_{\alpha}$ is $\frac{g_{\alpha}}{{\rm vol} (G)}$.

\end{enumerate}

Intuitively, $\mathcal{H}^T(G)$ is the amount of information required to determine the codeword of the vertex accessible from random walk with stationary distribution under the condition that the codeword of the starting vertex of the random walk is known. This intuition can be strictly proven. For this, we introduce an equivalent form of the structural entropy of a graph with respect to an encoding tree.

Suppose that $G=(V,E)$ is an undirected connected graph, and $T$ is an encoding tree of $G$.

 Consider a step of random walk with stationary distribution in $G$. Let $X$ and $Y$ be the random variables representing the codewords of the starting vertex $x$ and the arrival vertex $y$, respectively, of the random walk.

 Let $\alpha$ and $\beta$ be the codewords of $x$ and $y$, respectively. We consider the entropy of $\beta$ when we know $\alpha$. We denote this entropy by:

\begin{equation}\label{eqn:random walk}
\widetilde{H}(Y=\beta |X=\alpha).
\end{equation}

 Notice that the codeword $\alpha$ is a leaf node in $T$. By Lemma \ref{lem:coding tree property}, we know $T_{\delta}$ for all the nodes $\delta\subseteq\alpha$, i.e., the initial segments of $\alpha$ as strings.

Let $\gamma$ be the longest node $\delta\in T$ with $\delta\subseteq\alpha$ such that $y\in T_{\delta}$ holds. Then we know that $\gamma$ is an initial segment of the codeword $\beta$ of $y$ in $T$. To determine the codeword of $y$ in $T$, we only need to find the branch from $\gamma$ to a leaf node $\beta\in T$ such that $y\in T_{\beta}$. According to the analysis above, the information of $Y=\beta$ under the condition of knowing $X=\alpha$ is:

\begin{equation*}
\widetilde{H}(Y=\beta|X=\alpha)=-\sum\limits_{\substack{\delta\in T\\ \gamma\subset\delta\subseteq\beta}}\log_2\frac{{\rm vol}(\delta)}{{\rm vol}(\delta^{-})},
\end{equation*}
where $\gamma=\alpha\cap\beta$ is the node in $T$ at which $\alpha$ and $\beta$ branch in $T$, or $\gamma$ is the longest common initial segment of $\alpha$ and $\beta$.

Intuitively, $\widetilde{H}(Y=\beta |X=\alpha)$ is the amount of information to determine $\beta$ under the condition that $\alpha$ is known, where $\beta$ is the codeword of the vertex accessible from random walk from the vertex whose codeword is $\alpha$.

We notice that, only if  both $y\in T_{\delta}$ and $x\not\in T_{\delta}$ occur, we need to determine the codeword of $T_{\delta}$ in $T_{\delta^{-}}$, for which the amount of information required is $-\log_2\frac{{\rm vol}(\delta)}{{\rm vol}(\delta^{-})}$. So, intuitively, $\widetilde{H}(Y=\beta|X=\alpha)$ is the amount of information, in terms of the codeword of $T_{\delta}$ in $T_{\delta^{-}}$, required to determine the codeword of $y$ under the condition that the codeword of $x$ is known. Note that we use the codewords of nodes in the encoding tree to measure the amount of information. This is the reason why we use the notation $\widetilde{H}(\cdot)$ to distinguish from the classic conditional entropy notation $H(\cdot)$.

Define

\begin{equation*}
\widetilde{H}^T(G)=\frac{1}{{\rm vol} (G)}\sum\limits_{\substack{e=(x,\ y)\\ x,\ y\in V}}\widetilde{H}(Y=\beta|X=\alpha),
\end{equation*}
where $X$ is the codeword of vertex $x$, and $Y$ is the codeword of vertex $y$, accessible from random walk from $x$.

$\widetilde{H}^T(G)$ is then the average information for determining the codeword of the vertex accessible from random walk under the condition that the codeword of the starting vertex is known.

Our definition of $\mathcal{H}^T(G)$ in Definition \ref{def:structural entropy-T-O} is actually $\widetilde{H}^T(G)$.

\begin{lemma} \label{lem:Structural entropy-I} Let $G=(V,E)$ be a connected simple graph, and $T$ be an encoding tree of $G$. Then

\begin{equation}
\mathcal{H}^T(G)=\widetilde{H}^T(G).
\end{equation}

\end{lemma}
\begin{proof}
According to the definition of $\widetilde{H}(Y=\beta|X=\alpha)$, for every vertex $x$ and vertex $y$, for which there is an edge from $x$ to $y$, and $x$ and $y$ have codewords $\alpha$ and $\beta$ in $T$, respectively. Let $\gamma=\alpha\cap\beta$, that is, $\gamma$ is the longest initial segment of both $\alpha$ and $\beta$, then for every $\delta$, if $\gamma\subset\delta\subseteq\beta$, then the edge $(x,y)$ is in the cut from $\bar{T_{\delta}}$ to $T_{\delta}$. Therefore the edge from $x$ to $y$ contributes $-\log_2\frac{{\rm vol}(\delta)}{{\rm vol}(\delta^{-})}$ to $\widetilde{H}(Y=\beta|X=\alpha)$.

This ensures that

\begin{eqnarray*}
\widetilde{H}^T(G)&=&\frac{1}{{\rm vol} (G)}\sum\limits_{\substack{e=(x,\ y)\\ x,\ y\in V}}\widetilde{H}(Y=\beta|X=\alpha)\\
&=&-\frac{1}{{\rm vol}\ (G)}\sum\limits_{\substack{e=(x,\ y)\\  c (x)=\alpha,\  c(y)=\beta}}\sum\limits_{\substack{\delta\in T, \gamma=\alpha\cap\beta\\ \gamma\subset\delta\subseteq\beta}}\log_2\frac{{\rm vol}(\delta)}{{\rm vol}\ (\delta^{-})} \\
&=&-\frac{1}{{\rm vol}\ (G)}\sum\limits_{\delta\in T,\ \delta\not=\lambda}\sum\limits_{\substack{e=(x,y)\\ c(x)=\alpha,\ c(y)=\beta\\ \alpha\cap\beta=\gamma\subset\delta}}\log_2\frac{{\rm vol}(\delta)}{{\rm vol}(\delta^-)}\\
&=&-\frac{1}{{\rm vol}\ (G)}\sum\limits_{\delta\in T,\ \delta\not=\lambda}\sum\limits_{\substack{e=(x,y)\\ x\not\in T_{\delta},\ y\in T_{\delta}}}\log_2\frac{{\rm vol}(\delta)}{{\rm vol}(\delta^-)}\\
&=&-\sum\limits_{\alpha\in T,\ \alpha\not=\lambda}\frac{g_{\alpha}}{{\rm vol}(G)}\cdot \log_2\frac{{\rm vol}(\alpha)}{{\rm vol}(\alpha^{-})}\\
&=&\mathcal{H}^T(G),
\end{eqnarray*}
where $c(z)$ is the codeword of $z$ in $T$, $g_{\alpha}=|E(\bar{T_{\alpha}}, T_{\alpha})|$, that is, the number of edges from the complement of $T_{\alpha}$, i.e., $\bar{T_{\alpha}}$, to $T_{\alpha}$, ${\rm vol}(G)$ is the volume of $G$, that is, the total degree of vertices in $G$, ${\rm vol}(\beta)$ is the volume of the vertices set $T_{\beta}$, and $\alpha^{-}$ is the parent node of $\alpha$ in $T$.

\end{proof}

According to Lemma \ref{lem:Structural entropy-I}, $\mathcal{H}^T(G)$ measures the information required to determine the codeword given by $T$ of the vertex in $V$ that is accessible from random walk with stationary distribution in $G$, under the condition that the codeword of the starting vertex of the random walk is known.

Figure 1\footnote{The author would like to express thanks to his ph D student Qifu Hu for helping with the Figures 1, 2, 3, 4, 6.} explicitly explains the intuition of structural entropy of $G$ given in Lemma \ref{lem:Structural entropy-I}.

\begin{figure}
\centering

    \includegraphics[width=0.4\textwidth]{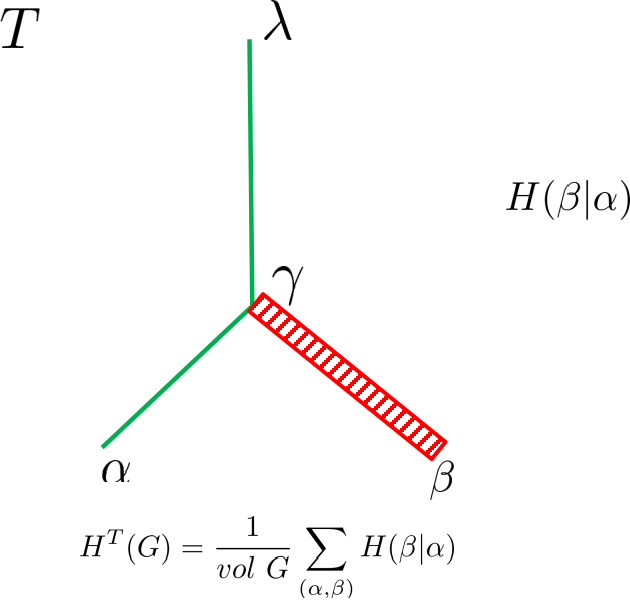}
    \caption{Structural Entropy. The notation $H(\beta |\ \alpha)$ is an abbreviation of $\widetilde{H}(Y=\beta |X=\alpha)$. Figure 1 shows that structural entropy corresponds to the quantity of uncertainty to determine the path from $\gamma$ and $\beta$. According to Figure 1, we notice that highly abstracted concepts such as entropy and information can be explicitly and intuitively represented. This fact itself is already very interesting and useful. }
    %\label{fig:Learning Model}
\end{figure}

In Figure 1, $\alpha$ is the codeword of a vertex $x$, and $\beta$ is the codeword of the vertex $y$ that is accessible from random walk from starting vertex $x$. Since the random walk starts from $x$, we assume that we have already known the codeword $\alpha$ of $x$. The advantage of encoding tree $T$ is that once we know $\alpha$, we know the path from the root node $\lambda$ to $\alpha$, meaning that we know the associated set $T_{\delta}$ for all the tree nodes $\delta$ on the path from $\lambda$ to $\alpha$. Suppose that $y$ is a neighbor of $x$ in $G$. Then we find the longest initial segment $\delta$ of $\alpha$ such that $y\in T_{\delta}$, denoted by $\gamma$. By the choice of $\gamma$, we know that the codeword of $y$ and $\alpha$ branch at $\gamma$. Since we have already known $\alpha$ and hence $\gamma$. To determine the codeword of $y$, we only need to determine the segment from $\gamma$ to the codeword of $y$. This is the amount of information $\widetilde{H}(Y=\beta |X=\alpha)$, where $X$ is the random variable representing the codeword of the starting vertex, and $Y$ represents the codeword of the vertex accessible from random walk. Then by Lemma \ref{lem:Structural entropy-I},
$\mathcal{H}^T(G)$ is the weighted average of all the $\widetilde{H}(Y=\beta |X=\alpha)$ over all the edges of $G$. 

From Figure 1, we know that an optimal encoding tree $T$ should ensure that for every edge $(x,y)$ of $G$, if $\alpha$ and $\beta$ are the codewords of $x$ and $y$, respectively in Figure 1, then there is only a short path between $\gamma$ and $\beta$, and the path is easy to determine, in the sense that, the uncertainty for determining $\beta$ once we know $\gamma$ is small.

[Remark: Principally speaking, 
Lemma \ref{lem:Structural entropy-I} itself could even be developed and extended to a general principle for network communications, in which an optimal encoding tree can be designed as a type of ``oracle" to guide the interactions and communications in massive communication networks. However, this needs a new project to develop.]

\subsection{Structural entropy}

\begin{definition}\label{def:structural-entropy-O} (Structural entropy of a graph, Li and Pan \cite{LP2016a}) Let $G=(V,E)$ be a graph.

\begin{enumerate}
\item [(1)] The structural entropy of $G$ is defined as

\begin{equation}
\mathcal{H}(G)=\min_{T}\{\mathcal{H}^T(G)\},
\end{equation}
where $T$ ranges over all the encoding trees of $G$.

[{\it Remark: Our structural entropy of a graph requires to find an encoding tree $T$ such that the $\mathcal{H}^T(G)$ in Equation (\ref{eqn:H-T-O}) is minimized.
Currently, there is no algorithm achieving the optimum structural entropy. However there are nearly linear time greedy algorithms for approximating the optimum encoding tree, with remarkable applications \cite{LP2016a, LYP2016, L+2018}.
}]

\item [(2)] For natural number $k$, the $k$-dimensional structural entropy of $G$ is defined as

\begin{equation}
\mathcal{H}^k(G)=\min_{T}\{\mathcal{H}^T(G)\},
\end{equation}
where $T$ ranges over all the encoding trees of $G$ of height at most $k$.

[{\it Remark}: This allows us to study the structural entropy of different dimensions. In practice, 2- or 3-dimensional structural information roughly corresponds to objects in the 2- or 3-dimensional space, respectively. ]

\item [(3)] Restricted structural entropy of a graph. For a type of encoding trees $\mathcal{T}$, we define the structural entropy of $G$ with respect to the type $\mathcal{T}$ to be the minimum of $\mathcal{H}^T(G)$ for all the encoding trees of type $\mathcal{T}$, written

\begin{equation}
\mathcal{H}^{\mathcal{T}}(G)=\min_{T}\{\mathcal{H}^T(G)\},
\end{equation}
where $T$ ranges over all the encoding trees in $\mathcal{T}$.
\end{enumerate}

\end{definition}

To better understand Definition \ref{def:structural-entropy-O}, We look at Figures 2 and 3.

\begin{figure}\label{fig:structural information decoding}
\centering

    \includegraphics[width=0.9\textwidth]{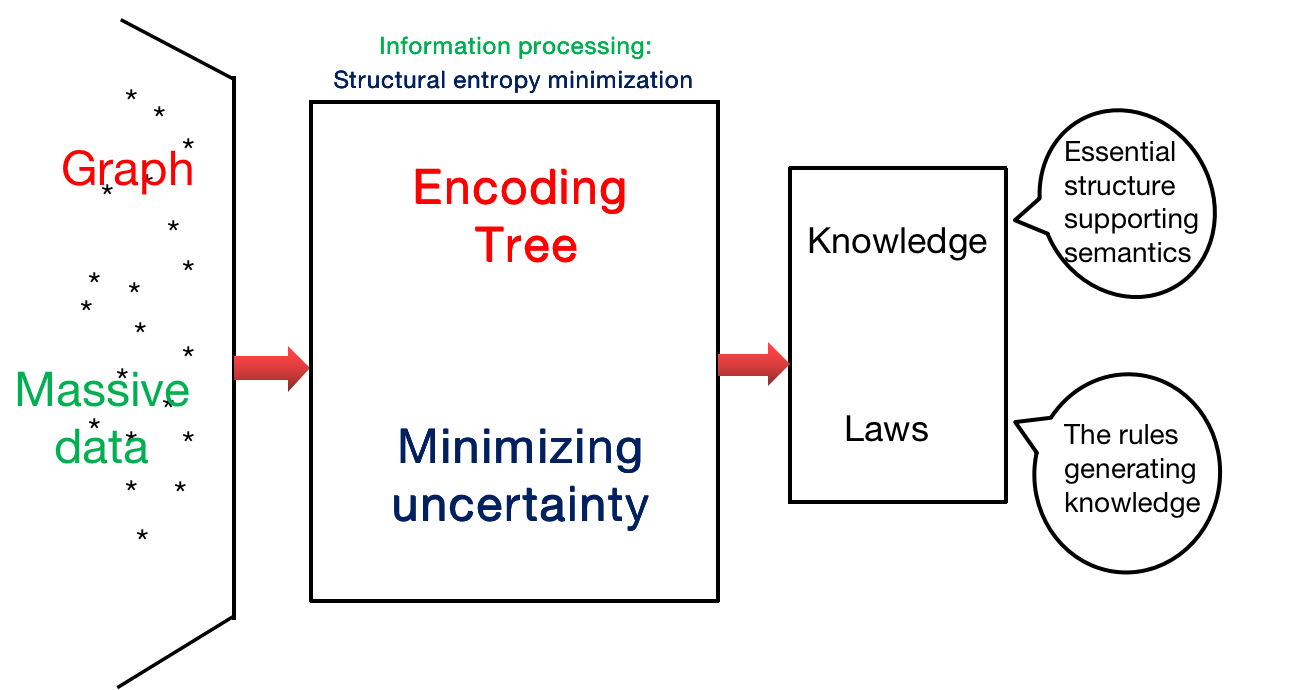}
    \caption{Structural information decoding}
    %\label{fig:Learning Model}
\end{figure}

Figure 2 describes the procedure of the encoding/decoding by finding the encoding tree that minimizes the structural entropy of a graph. 

According to Figure 2, the encoding/decoding of a graph $G$ proceeds as follows:

\begin{enumerate}
\item [(1)] Find an encoding tree $T$ of a given type $\mathcal{T}$ such that the structural entropy $\mathcal{H}^T(G)$ is minimized, or approximately minimized.

\item [(2)] Due to the fact that the structural entropy of $G$ given by $T$ is minimized, the encoding tree $T$ must be the encoding of an essential structure of $G$ that supports a semantics of $G$.

\item [(3)] $T$ is hence such a syntax of $G$ supporting the semantics of $G$. So by interpreting $T$, we are able to find the knowledge of $G$, referred to as a {\it knowledge tree} of $G$, written $KT(G)$.

\item [(4)] Since the encoding tree $T$ found this way is an essential structure of $G$ and $KT(G)$ is the knowledge tree of $G$, from both $T$ and $KT(G)$, we are able to extract the rules that generate $T$ and $KT(G)$. This set of rules is regarded as laws of $G$. 

\item [(5)] Figure 2 shows that structural entropy minimization is a principle for information processing, in which encoding tree is both an encoder and a decoder.

\item [(6)] The most important feature of encoding tree is that encoding trees are lossless encoders of graphs, and that trees are highly efficient data structures, supporting efficient algorithms.

\item [(7)] The encoding/decoding using the encoding tree of graphs implies that encoding not only eliminates uncertainty embedded in a complex system, but also provides efficient data structures for algorithms. This suggests a new direction of the combination of coding theory and algorithms to study the role of encoding in the design of algorithms.
\item [(8)] Structuring of unstructured massive dataset is a principle for data analysis.
\end{enumerate}

Figure 3 below shows that due to the definition of the structural entropy in Definition \ref{def:structural-entropy-O}. The optimization of the structural entropy could be restricted to various types of encoding trees. For each of such a type, there is a new optimization problem. All these optimization problems lead to new optimization problems. Due to the definition of the structural entropy, these optimization problems have new characters. On one hand, the goal is a sum of log functions, which is highly similar to the convex optimizations. However, the objects are graphs, which are combinatorial objects. This new feature makes the optimization problems extremely interesting. In fact, in real world applications, although the objects are combinatorial, the strategies of convex optimization usually work perfectly well in both efficiency and quality.
Therefore, the structural entropies in Definition \ref{def:structural-entropy-O} lead to various optimization problems, referred to as {\it information optimization problems}. Information optimization is hence a new direction between convex optimization and combinatorial optimization, calling for a new theory.

\begin{figure}
\centering

    \includegraphics[width=0.7\textwidth]{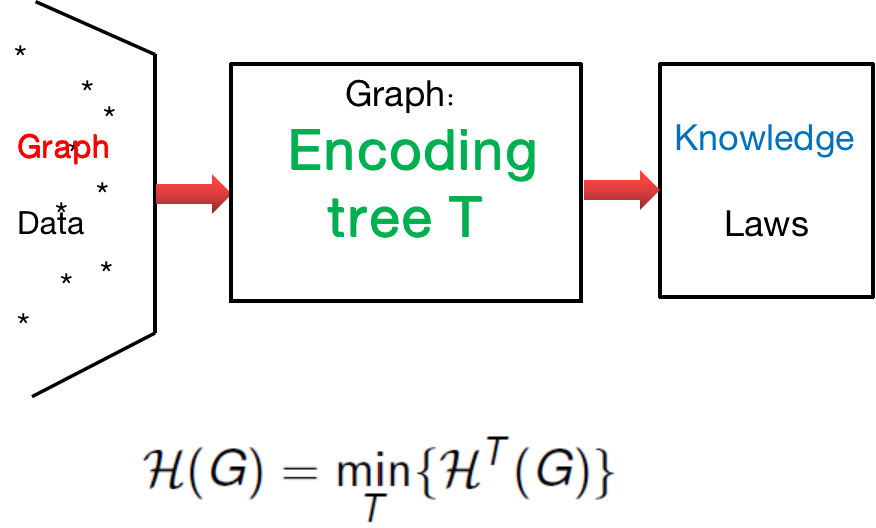}
    \caption{Information Optimization: Encoding eliminates uncertainty. This leads to a new direction for optimization that is between convex optimization and combinatorial optimization.}
    %\label{fig:Learning Model}
\end{figure}

The metric $\mathcal{H}(G)$ has the following intuitions:

\begin{itemize}
\item Intuitively speaking, the structural entropy $\mathcal{H}(G)$ of $G$ is the least amount of information required to determine the codeword of the vertex in an encoding tree that is accessible from random walk with stationary distribution in $G$, under the condition that the codeword of the vertex at which random walk starts is known.

\item Mathematically speaking, the structural entropy $\mathcal{H}(G)$ of $G$ is essentially the {\it intrinsic information hidden in $G$} that cannot be decoded by any encoding tree or any lossless encoding of $G$.

\item The structural entropy $\mathcal{H}(G)$ of $G$ is the information that {\it determines} and {\it decodes} the encoding tree $T$ of $G$ that minimizes the uncertainty in positioning the vertex that is accessible from random walk in graph $G$, when an encoding tree $T$ is given as an ``oracle".

Therefore, $\mathcal{H}(G)$ is not only a measure of structural information, but decodes the structure of $G$ that minimizes the uncertainty in the communications in the graph, which can be regarded as the ``essential structure" (or {\it decoder of $G$}, for short) of the graph.

\item Due to the fact that the encoding tree $T$ minimizes the uncertainty hidden in $G$, $T$ eliminates the uncertainty embedded in $G$. $T$ is both an {\it encoder} and a {\it decoder} of $G$.

\item $T$ is a syntax structure of $G$ finding from the syntax system of $G$. Since $T$ minimizes the uncertainty hidden in $G$, $T$ supports a semantics of $G$. The semantics of $G$ interpreted from $T$ is a {\it knowledge tree}, written KT, of $G$.
\item The knowledge tree KT of $G$ provides a {\it tree of abstractions} of system $G$. This means that a decoder $T$ of $G$ provides not only a knowledge tree of $G$, but also an {\it abstracting tree} consisting of a hierarchical system of abstractions of $G$. This gives rise to not only an abstracting of $G$, but also a hierarchical system of abstracting of $G$, corresponding to high-level abstractions of $G$. This observation is crucial for our structural information learning machinery (SiLeM).

\item A decoder and the corresponding knowledge tree determines and generates a {\it tree of abstractions} that can be regarded as the basics of intuitive reasoning in learning.

\end{itemize}

The $k$-dimensional structural entropy $\mathcal{H}^k(G)$ of $G$ has the similar intuitions as above.

We remark that the structural entropy has a rich theory with remarkable applications, more details are referred to \cite{LP2016a, LYP2016, L+2018}.

%The $k$-dimensional structural entropy of a graph gives rise to a hierarchy of the structural entropy of the graph.

\subsection{A general lower bound of structural entropy}\label{subsec:lowerbound}

To establish our structural information learning theory, we prove a new lower bound of the structural entropy of graphs.

 Given a graph $G=(V,E)$, and a subset $S$ of $V$, the conductance of
$S$ in $G$ is given by

\begin{equation} \label{eqn:phisubset}
\Phi (S)=\frac{|E(S,\bar{S})|}{\min\{ {\rm vol}(S), {\rm vol
}(\bar{S})\}},
\end{equation}

\noindent where $E(S,\bar{S})$ is the set of edges with one
endpoint in $S$ and the other in the complement of $S$, i.e.
$\bar{S}$, and ${\rm vol}(X)$ is the sum of degrees $d_x$ for all
$x\in X$. The conductance of $G$ is defined to be the minimum of
$\Phi(S)$ over all subsets $S$'s, that is:

\begin{equation} \label{eqn:phiG}
\Phi (G)=\min\limits_{S\subset V}\{\Phi (S)\}.
\end{equation}

\begin{theorem} (Lower bound of structural entropy of a graph)\label{thm:SEprinciple-New} For any undirected and connected graph $G$, the structural entropy of $G$ satisfies the following lower bound:

\begin{equation}
\mathcal{H}(G)\geq\Phi (G)\cdot (\mathcal{H}^1(G)-1),
\end{equation}
\noindent where $\Phi (G)$ is the conductance of $G$, and $\mathcal{H}^1(G)$ is the one-dimensional structural entropy of $G$.
\end{theorem}
\begin{proof} Let $T$ be an encoding tree of $G$.

Define

\begin{equation}
A=\{\alpha\ |\ \alpha\in T, \alpha\not=\lambda, V_{\alpha}>\frac{1}{2}V_{\lambda}\}.
\end{equation}

\begin{equation}
B=\{\alpha\ |\ \alpha\in T, \alpha\not=\lambda, V_{\alpha}\leq\frac{1}{2}V_{\lambda}\}.
\end{equation}

For every tree node $\alpha\in A$, we have

\begin{equation*}
\Phi_{\alpha}=\Phi (T_{\alpha})=\frac{g_{\alpha}}{V_{\lambda}-V_{\alpha}}\geq \Phi (G).
\end{equation*}

Therefore,

$$g_{\alpha}\geq \Phi (G)\cdot (V_{\lambda}-V_{\alpha}).$$

Based on this, we have

\begin{eqnarray}\label{formula:A}
&&-\sum\limits_{\alpha\in A}\frac{V_{\alpha}}{V_{\lambda}}\cdot\frac{g_{\alpha}}{V_{\alpha}}\cdot\log_2\frac{V_{\alpha}}{V_{\alpha^-}} \nonumber \\
&\geq& -\sum\limits_{\alpha\in A}\frac{V_{\alpha}}{V_{\lambda}}\cdot\frac{\Phi(G)[V_{\lambda}-V_{\alpha}]}{V_{\alpha}}\cdot\log_2\frac{V_{\alpha}}{V_{\alpha^-}}\nonumber\\
&=&\Phi (G)\cdot [-\sum\limits_{\alpha\in A}\frac{V_{\alpha}}{V_{\lambda}}\cdot (\frac{V_{\lambda}}{V_{\alpha}}-1)\cdot\log_2\frac{V_{\alpha}}{V_{\alpha^-}}] \nonumber\\
&=&\Phi (G)[-\sum\limits_{\alpha\in A}(1-\frac{V_{\alpha}}{V_{\lambda}})\cdot\log_2\frac{V_{\alpha}}{V_{\alpha^-}}] \nonumber\\
&=&\Phi (G)\cdot [-\sum\limits_{\alpha\in A}\frac{V_{\alpha}}{V_{\lambda}}\cdot\log_2\frac{V_{\alpha}}{V_{\alpha^-}}-\sum\limits_{\alpha\in A}(1-\frac{2V_{\alpha}}{V_{\lambda}})\cdot\log_2\frac{V_{\alpha}}{V_{\alpha^-}}].
\end{eqnarray}

For every tree node $\alpha\in B$, we have

\begin{equation}
\Phi_{\alpha}=\Phi (T_{\alpha})=\frac{g_{\alpha}}{V_{\alpha}}\geq\phi (G).
\end{equation}

Therefore,

\begin{equation}\label{formula:B}
-\sum\limits_{\alpha\in B}\frac{V_{\alpha}}{V_{\lambda}}\cdot\frac{g_{\alpha}}{V_{\alpha}}\cdot\log_2\frac{V_{\alpha}}{V_{\alpha^-}}\geq\Phi (G)\cdot [-\sum\limits_{\alpha\in B}\frac{V_{\alpha}}{V_{\lambda}}\cdot\log_2\frac{V_{\alpha}}{V_{\alpha^-}}].
\end{equation}

Let

\begin{equation}
\Delta= \sum\limits_{\alpha\in A}(1-\frac{2V_{\alpha}}{V_{\lambda}})\cdot\log_2\frac{V_{\alpha}}{V_{\alpha^-}}.
\end{equation}

Since for every $\alpha\in A$, $\frac{1}{2}V_{\lambda}<V_{\alpha}<V_{\lambda}$, so

$$0<\frac{2V_{\alpha}}{V_{\lambda}}-1<1.$$

Therefore,

\begin{eqnarray}
\Delta&=&\sum\limits_{\alpha\in A}(1-\frac{2V_{\alpha}}{V_{\lambda}})\cdot\log_2\frac{V_{\alpha}}{V_{\alpha^-}} \nonumber\\
&=&-\sum\limits_{\alpha\in A}(\frac{2V_{\alpha}}{V_{\lambda}}-1)\cdot\log_2\frac{V_{\alpha}}{V_{\alpha^-}}\nonumber\\
&<&-\sum\limits_{\alpha\in A}\log_2\frac{V_{\alpha}}{V_{\alpha^-}}.
\end{eqnarray}

By the definition of $A$ and $B$, we have that at every level of the coding tree, there is at most one node in $A$, that for every $\alpha\in A$, the parent node $\alpha^-$ is either in $A$ or equal to $\lambda$ and that every child of a node in $B$ must be also in $B$. Therefore, all the tree nodes $\alpha$ in $A$ are in a single branch of the coding tree $T$.

Suppose that
$\alpha_1\subset\alpha_2\subset\cdots\subset\alpha_l$
are all the nodes $\alpha$ in $A$.

Then:

\begin{equation}
\lambda=\alpha_0\subset\alpha_1\subset\alpha_2\subset\cdots\subset\alpha_l.
\end{equation}

Therefore,

\begin{eqnarray}
\Delta&<&-\sum\limits_{\alpha\in A}\log_2\frac{V_{\alpha}}{V_{\alpha^-}} \nonumber\\
&=&-\log_2\prod\limits_{i=1}^l\frac{V_i}{V_{i-1}},\ V_i=V_{\alpha_i} \nonumber\\
&=&-\log\frac{V_l}{V_0}=\log\frac{V_{\lambda}}{V_l}.
\end{eqnarray}

Since $\frac{1}{2}V_{\lambda}<V_l<V_{\lambda}$,
\begin{equation}\label{formula:Delta}
0<\Delta<1.
\end{equation}

According to the inequalities in (\ref{formula:A}), (\ref{formula:B}) and the analysis of $\Delta$, we have:

\begin{eqnarray}
\mathcal{H}^T(G)&=&-\sum\limits_{\alpha\in T, \alpha\not=\lambda}\frac{g_{\alpha}}{V_{\lambda}}\log_2\frac{V_{\alpha}}{V_{\alpha^-}} \nonumber\\
&=&-\sum\limits_{\alpha\in T, \alpha\not=\lambda}\frac{V_{\alpha}}{V_{\lambda}}\cdot\frac{g_{\alpha}}{V_{\alpha}}\cdot\log_2\frac{V_{\alpha}}{V_{\alpha^-}} \nonumber\\
&=& -\sum\limits_{\alpha\in A}\frac{V_{\alpha}}{V_{\lambda}}\cdot\frac{g_{\alpha}}{V_{\alpha}}\cdot\log_2\frac{V_{\alpha}}{V_{\alpha^-}} -\sum\limits_{\alpha\in B}\frac{V_{\alpha}}{V_{\lambda}}\cdot\frac{g_{\alpha}}{V_{\alpha}}\cdot\log_2\frac{V_{\alpha}}{V_{\alpha^-}} \nonumber\\
&\geq&\Phi (G)\cdot [-\sum\limits_{\alpha\in T, \alpha\not=\lambda}\frac{V_{\alpha}}{V_{\lambda}}\cdot\log_2\frac{V_{\alpha}}{V_{\alpha^-}}-\sum\limits_{\alpha\in A}(1-\frac{2V_{\alpha}}{V_{\lambda}})\cdot\log_2\frac{V_{\alpha}}{V_{\alpha^-}}] \nonumber\\
&=&\Phi (G)\cdot (\mathcal{H}^1(G)-\Delta)\nonumber \\
&>&\Phi (G)\cdot (\mathcal{H}^1(G)-1).
\end{eqnarray}

Theorem \ref{thm:SEprinciple-New} follows.
\end{proof}

Our structural information learning machines will be developed based on the structural information theory. Now we finished the necessary theory of structural information for us to develop our learning machines.

We notice that, the structural information learning machines need Shannon's theory as well. For this, we will build the Shannon metric by using the encoding trees to unify the fundamentals of both Shannon's information theory and ours structural information theory.

\section{Structural Entropy Naturally Extends the Shannon Entropy}\label{sec:extension-Shannon}

%\subsection{Module function of a graph}

The structural entropies of a graph in Definitions \ref{def:structural entropy-T-O} and \ref{def:structural-entropy-O} are often misunderstood as the average length of the codeword of the vertex that is accessible from random walk with stationary distribution in the graph. We argue that this is not the case.

\subsection{Structural entropy with a module function}

To better understand the question, we introduce a variation of the structural entropy. It depends on a module function of a graph.

\begin{definition}\label{def:module function} (Module function) Let $G=(V,E)$ be a connected graph. Let ${\rm vol}(G)$ be the volume of $G$. A module function of $G$ is a function $g$ of the form:

\begin{equation}
g:\ 2^V\ \rightarrow \{0, 1, \cdots, {\rm vol}(G)\}.
\end{equation}
\end{definition}

We define the structural entropy of a graph with a module function as follows.

\begin{definition} \label{def:structural entropy-T} (Structural entropy of a graph with a module function by an encoding tree) Let $G=(V,E)$ be a graph, $g$ be a module function of $G$, and $T$ be an encoding tree of $G$.
We define the structural entropy of $G$ with module function $g$ by encoding tree $T$ as follows:

\begin{equation}\label{eqn:H-T}
\mathcal{H}^{T}_g(G)=-\sum\limits_{\alpha\not=\lambda, \alpha\in\ T}\frac{g(T_\alpha)}{{\rm vol}(G)}\cdot \log_2\frac{{\rm vol}(\alpha)}{{\rm vol}(\alpha^{-})},
\end{equation}
where ${\rm vol}(G)$ is the volume of $G$, ${\rm vol}(\beta)$ is the volume of the vertices set $T_{\beta}$, and $\alpha^{-}$ is the parent node of $\alpha$ in $T$.

\end{definition}

%\subsection{structural entropy with a module function}

\begin{definition}\label{def:structural-entropy} (Structural entropy of a graph with a module function) Let $G=(V,E)$ be a graph, and $g$ be a module function of $G$.

\begin{enumerate}
\item [(1)] The structural entropy of $G$ with module function $g$ is defined as

\begin{equation}
\mathcal{H}_g(G)=\min_{T}\{\mathcal{H}^T_g(G)\},
\end{equation}
where $T$ ranges over all the encoding trees of $G$.

\item [(2)] For natural number $k$, the $k$-dimensional structural entropy of $G$ with module function $g$ is defined as

\begin{equation}
\mathcal{H}^k_g(G)=\min_{T}\{\mathcal{H}^T_g(G)\},
\end{equation}
where $T$ ranges over all the encoding trees of $G$ of height less than or equal to $k$.

\end{enumerate}

\end{definition}

The formula $\mathcal{H}^{T}_g(G)=-\sum\limits_{\alpha\not=\lambda, \alpha\in\ T}\frac{g(T_\alpha)}{{\rm vol}(G)}\cdot \log_2\frac{{\rm vol}(\alpha)}{{\rm vol}(\alpha^{-})}$ is a generalization of $\mathcal{H}^T(G)$ in Definition \ref{def:structural entropy-T-O} with the function $g$ here being an arbitrarily given module function, while the function $g$ in Definition \ref{def:structural entropy-T-O} is the cut module function, that is, the number of edges in the cut.

In Definitions \ref{def:structural entropy-T} and \ref{def:structural-entropy}, the structural entropy of graph $G$ depends on a choice of a module function $g$. It is possible that there are many interesting choices for the module function $g$. We list a few of these as example:

\begin{enumerate}
\item [(i)] For a subset $X$ of vertices $V$, $g(X)$ is the volume of $X$. In this case, $g$ is called the {\it volume module function}.
\item [(ii)] For each subset $X$ of $V$, $g(X)$ is the weights in the cut $(X,\bar{X})$ in $G$. In this case, we say that $g$ is the {\it cut module function}.
\item [(iii)] For a directed graph $G$ and for each subset $X$ of $V$, $g(X)$ is the weights of the flow from $\bar{X}$ to $X$. In this case, we call $g$ the {\it flow module function}.

For directed graphs, the flow module function would be essential to the structural entropy of the graphs.
\end{enumerate}

In particular, there are module functions with additivity, with which the structural entropy collapses to the Shannon entropy.

\begin{definition}\label{def:additive function} (Additive module function) Let $G=(V,E)$ be a connected, simple graph with $n$ vertices and $m$ edges, and $g$ be a module function of $G$.
 We say that $g$ is an additive module function if for any disjoint sets $X$ and $Y$ of $V$,

\begin{equation}
g(X\cup Y)=g(X)+g(Y).
\end{equation}

\end{definition}

\begin{theorem} \label{thm:additive function theorem} \label{thm:additivity function theorem} (Structural entropy of a graph with an additive function) Let $G=(V,E)$ be a connected, simple graph with $n$ vertices, and $m$ edges, and let $g$ be an additive module function of $G$. For any encoding tree $T$ of $G$, if $g$ satisfies the boundary condition
\begin{equation}	
g(T_{\alpha})=d_{\alpha},\ \text{if\ \rm $\alpha$\ is\ a\ leaf},
\end{equation}	
then
\begin{equation}
\mathcal{H}^T_g(G)=-\sum\limits_{i=1}^n\frac{d_i}{2m}\cdot\log_2\frac{d_i}{2m},
\end{equation}
where $d_{\alpha}$ is the degree of the vertex with codeword $\alpha$, and $d_i$ is the degree of vertex $i$ in $G$.
\end{theorem}
\begin{proof} By Definition \ref{def:structural entropy-T}, noting that for every $\alpha\in T$, let $g_{\alpha}=g(T_{\alpha})$ and $V_{\alpha}={\rm vol}(\alpha)$, we have:

\begin{eqnarray}
\mathcal{H}^T_g(G)&=&-\sum\limits_{\alpha\not=\lambda, \alpha\in\ T}
\frac{g_{\alpha}}{2m}\cdot \log_2\frac{V_{\alpha}}{V_{\alpha^{-}}} \nonumber\\
&=&-\sum\limits_{\alpha\not=\lambda, \alpha\in\ T}
\frac{g_{\alpha}}{2m}\cdot \log_2V_{\alpha} + \sum\limits_{\substack{\alpha\not=\lambda\\ \alpha\in\ T}}
\frac{g_{\alpha}}{2m}\cdot \log_2V_{\alpha^{-}} \nonumber\\
&=&-\sum\limits_{\alpha\not=\lambda, \alpha\in\ T}
\frac{g_{\alpha}}{2m}\cdot \log_2V_{\alpha} + (\sum\limits_{\substack{\alpha\in T\\ \text{non-leaf}}}
\frac{g_{\alpha}}{2m}\cdot \log_2V_{\alpha}+\log_2(2m)), \ \text{by the additivity of $g$} \nonumber\\
&=&-\sum\limits_{i=1}^n\frac{d_i}{2m}\cdot\log_2\frac{d_i}{2m}.
\end{eqnarray}

The theorem follows.
\end{proof}

\begin{definition}\label{def:length of random accessible vertex} (The length of the vertex accessible from random walk with stationary distribution)
For a connected, simple graph $G=(V,E)$ of $n$ vertices and $m$ edges. Let $g$ be the volume module function of $G$ defined as: for any set $X$ of vertices $V$, $g(X)$ is the volume of $X$.
Suppose that $T$ is an encoding tree of $G$. Then:

\begin{equation}\label{eqn:Length-G-T}
H^T_g(G)=-\sum\limits_{\substack{\alpha\in T\\ \alpha\not=\lambda}}\frac{V_{\alpha}}{2m}\cdot\log_2\frac{V_{\alpha}}{V_{\alpha^{-}}},
\end{equation}
where $V_{\beta}$ is the volume of $T_{\beta}$, $\alpha^{-}$ is the parent node of $\alpha$ in $T$.

\end{definition}

{\it Remarks: In this case, ${H}^T_g(G)$ is the amount of information required to describe the codeword in $T$ of the vertex that is accessible from random walk with stationary distribution in $G$, and is a lower bound of the average length of the codeword (in $T$) of the vertex that is accessible from random walk with stationary distribution in $G$.}

\begin{corollary} \label{thm:length of graph T}\label{thm:additivity theorem} For any connected and simple graph $G=(V,E)$ with $n$ vertices and $m$ edges. For the module function $g(X)=\sum\limits_{x\in X}d_x$, where $d_x$ is the degree of $x$ in $G$, and for any encoding tree $T$ of $G$,

\begin{eqnarray}
{H}^T_g(G)&=&-\sum\limits_{i=1}^n\frac{d_i}{2m}\cdot\log_2\frac{d_i}{2m} \nonumber\\
&=&\mathcal{H}^1(G),
\end{eqnarray}
where $d_i$ is the degree of vertex $i$ in $G$, $\mathcal{H}^1(G)$ is the one-dimensional structural entropy of $G$ \cite{LP2016a}.

\end{corollary}
\begin{proof} Note that for any non-leaf node $\alpha\in T$, $V_{\alpha}=\sum\limits_{\beta\in T, \beta^{-}=\alpha}V_{\beta}$, that is, $V_{\alpha}$ is an additive module function. The result follows from Theorem \ref{thm:additive function theorem}.
\end{proof}

Corollary \ref{thm:length of graph T} shows that

\begin{itemize}

\item The information to {\it describe} the codeword of an encoding tree of the vertex that is accessible from random walk with stationary distribution in $G$ is independent of any encoding tree $T$ of $G$, and
\item The {\it minimum average length}, written $L(G)$, of the codeword in an encoding tree of the vertex that is accessible from random walk with stationary distribution is greater than or equal to (or lower bounded by) the one-dimensional structural entropy $\mathcal{H}^1(G)$ \cite{LP2016a}, or the Shannon entropy of the degree distribution of the graph. This means that

    \begin{equation}
    L(G)=\Omega (\log_2n),
    \end{equation}
where $n$ is the number of vertices in $G$.

This property is in sharp contrast to the structural entropy. In fact, there are many graphs $G$ such that the two-dimensional structural entropy $\mathcal{H}^2(G)=O(\log_2\log_2n)$, referred to \cite{LP2016a}.
\end{itemize}

The proof of Theorem \ref{thm:additive function theorem} also shows the reason why the structural entropies in Definitions \ref{def:structural entropy-T-O} and \ref{def:structural-entropy-O} depend on the encoding trees of a graph. The reason is that, the cut module function $g$ in Definition \ref{def:structural entropy-T-O} fails to have the additivity, since for any two disjoint vertex sets $X$ and $Y$, if there are edges between $X$ and $Y$, then $g(X\cup Y)<g(X)+g(Y)$. This ensures that the structural entropy $\mathcal{H}^T(G)$ in Definition \ref{def:structural entropy-T-O} depends on the encoding tree $T$ of $G$. For this reason, the structural entropy provides the foundation for a new direction of information theory with rich theory and remarkable applications \cite{LYP2016, L+2018}.

\subsection{Shannon entropy is independent of encoding trees}

Given a probability distribution $p=(p_1, p_2,\cdots, p_n)$, the structural entropy can be naturally defined on $p$ as follows:

\begin{definition} (Encoding tree of $p$) An {\it encoding tree of $p$} is a rooted tree as before. The root node $\lambda$ (the empty string) is associated with the set of all the items $\{1, 2,\cdots, n\}$. Every tree node $\alpha$ is associated with a subset $T_{\alpha}$ of $\{1,2,\cdots,n\}$. For every tree node $\alpha$, if $\beta_1,\beta_2,\cdots,\beta_l$ are all the children of $\alpha$, then $\{T_{\beta_1}, T_{\beta_2},\cdots, T_{\beta_l}\}$ is a partition of $T_{\alpha}$. Of course, every leaf node $\gamma$ in the tree is associated with a singleton $\{i\}$ for some $i\in\{1,2,\cdots,n\}$, that is, $T_{\gamma}=\{i\}$. If $T_{\gamma}=\{i\}$, then we say that $\gamma$ is the {\it codeword of $i$}.

\end{definition}

\begin{definition} (The structural entropy of $p$ given by an encoding tree $T$ of $p$) Let $p=(p_1,p_2,\cdots, p_n)$ be a probability distribution and $T$ be an encoding tree of $p$. Then {\it the structural entropy of $p$ given by $T$} is defined as

\begin{equation}\label{eqn:structural entropy of p}
\mathcal{H}^T(p)=-\sum\limits_{\alpha\in T, \alpha\not=\lambda}\frac{g_{\alpha}}{V_{\lambda}}\cdot\log_2\frac{V_{\alpha}}{V_{\alpha^{-}}},
\end{equation}
where for every tree node $\alpha$, $g_{\alpha}=V_{\alpha}=\sum\limits_{i\in T_{\alpha}}p_i$.

\end{definition}

\begin{theorem} (Shannon entropy is independent of encoding tree)\label{thm:Shannon entropy} Given a probability distribution $p=(p_1,p_2,\cdots,p_n)$, let $T$ be an encoding tree of $p$. Then:

\begin{equation}
\mathcal{H}^T(p)=H(p),
\end{equation}
where $H(p)$ is the Shannon entropy of $p$.
\end{theorem}
\begin{proof}
By the proof of Theorem \ref{thm:additivity function theorem}.
\end{proof}

Theorem \ref{thm:Shannon entropy} ensures that the structural entropy on the unstructured probability distribution degrades to the Shannon entropy. Therefore, the structural entropy is a natural extension of Shannon's entropy from unstructured data to structured systems.

Combining the structural information theory and Shannon's information theory together allows us to define the new concepts of compressing information and decoding information of graphs, and to establish the principles of graph compressing and graph decoding.
The new principles build the foundation of our structural information learning machines.

\section{Compressing Information and Decoding Information Principle}\label{sec:structuredocoding}

By definition, the structural entropy $\mathcal{H}(G)$ of $G$ is {\it the intrinsic information hidden in $G$}. Therefore,
$\mathcal{H}(G)$ is the amount of information deeply hidden in $G$ that cannot be decoded eventually, anyway.

In the classical information theory, we need to measure the compression ratio of a random variable or a probability distribution, interpreted as data. In structural information theory, we need to answer the question of how much information embedded in a graph that can be compressed, and that can be decoded. In this section, we answer these questions by using the structural entropy of graphs.

Let $G$ be a connected and undirected graph. We have shown that the one-dimensional structural entropy of $G$ is the Shannon entropy of the degree distribution of $G$.
For this reason, we define:

\begin{definition} (Shannon entropy of a graph)\label{def:Shannonentropy} Let $G$ be an undirected and connected graph.
We define the Shannon entropy of $G$ to be the one-dimensional structural entropy of $G$, written as

\begin{equation}
H(G)=\mathcal{H}^1(G).
\end{equation}
\end{definition}

The Shannon entropy of $G$, or the one-dimensional structural entropy of $G$ can be understood as the {\it amount of uncertainty} that is embedded in $G$.

It has been a grand challenge to define the compressing information of a graph. Here we define such a metric. It is defined by using the encoding trees of the graph.

\subsection{Compressing information of a graph}\label{subsec:compressing}

\begin{definition} (Compressing information of a graph given by an encoding tree) 
\label{def:compressing information given by an encoding tree} Given an undirected and connected graph $G=(V,E)$, let $T$ be an encoding tree of $G$. We define the compressing information of $G$ given by $T$ as

\begin{equation}\label{eqn:compressing-tree}
\mathcal{C}^T(G)=-\sum\limits_{\substack{\alpha\in T\\ \alpha\not=\lambda}}\frac{V_{\alpha}-g_{\alpha}}{V_{\lambda}}\log_2\frac{V_{\alpha}}{V_{\alpha^-}},
\end{equation}
where $g_{\alpha}$ is the number of edges from vertices outside of $T_{\alpha}$ to vertices in $T_{\alpha}$, $V_{\beta}$ is the volume of $T_{\beta}$, and $\alpha^-$ is the parent node of $\alpha$ in the encoding tree $T$.
\end{definition}

The intuition of Definition \ref{def:compressing information given by an encoding tree} is as follows:

\begin{enumerate}

\item [(i)] We interpret the encoding tree $T$ as an {\it encoder} of $G$.

\item [(ii)] In Equation (\ref{eqn:compressing-tree}),
\begin{enumerate}

\item [(a)] $\frac{V_{\alpha}-g_{\alpha}}{V_{\lambda}}$ is the probability that random walk keeps staying in the same module $T_{\alpha}$,
\item  [(b)] $-\log_2\frac{V_{\alpha}}{V_{\alpha^-}}$ is the information of $\alpha$ within $\alpha^-$, and
\item [(c)] $-\frac{V_{\alpha}-g_{\alpha}}{V_{\lambda}}\log_2\frac{V_{\alpha}}{V_{\alpha^-}}$ measures the information that random walk in $G$ keeps staying in the same module $T_{\alpha}$.
\end{enumerate}

\item [(iii)] According to (ii) above, if $\mathcal{C}^T(G)$ is large, then the uncertainty of random walk in $G$ is reduced by the encoding tree $T$, and
\item [(iv)] By (iii) above, $T$ is regarded an a {\it compressor} of $G$.
\end{enumerate}

Similar to the intuitive interpretation of the notion of structural entropy in Lemma \ref{lem:Structural entropy-I}, the compressing information of $G$ given by an encoding tree $T$ can be interpreted as the average of {\it mutual information} of the codewords of edges taken by a random walk in $G$. This intuition is given in Figure 4 below.

We can prove the intuition mathematically.

 Consider a step of random walk with stationary distribution in $G$. Let $X$ and $Y$ be the random variables representing the codewords of the starting vertex $x$ and the arrival vertex $y$, respectively, of the random walk.

 Suppose that $\alpha$ and $\beta$ are the codewords of $x$ and $y$, respectively.

 We consider the {\it mutual information} between $X=\alpha$ and $Y=\beta$, denoted by:

\begin{equation}\label{eqn:mutual-information-random walk}
\widetilde{I}(X=\alpha;Y=\beta).
\end{equation}

 Notice that the codeword $\alpha$ is a leaf node in $T$. By Lemma \ref{lem:coding tree property}, if we know $\alpha$, then we know $T_{\delta}$ for all the nodes $\delta\subseteq\alpha$, i.e., the initial segments of $\alpha$ as strings.

Let $\gamma$ be the longest node $\delta\in T$ such that both $\delta\subseteq\alpha$ and $\delta\subseteq\beta$ hold. We know that once we know $\alpha$, we have already known $\gamma$, and that $\gamma$ is the part of $\beta$ that we have already known. Therefore $\gamma$ is the part shared by $\alpha$ and $\beta$. This means that the mutual information between $\alpha$ and $\beta$ is the information required to determine $\gamma$.
 
 According to the analysis above, the mutual information of $X=\alpha$ and $Y=\beta$ is:

\begin{equation*}
\widetilde{I}(X=\alpha;Y=\beta)=-\sum\limits_{\substack{\delta\in T\\ \lambda\subset\delta\subseteq\gamma}}\log_2\frac{{\rm vol}(\delta)}{{\rm vol}(\delta^{-})},
\end{equation*}
where $\gamma=\alpha\cap\beta$ is the node in $T$ at which $\alpha$ and $\beta$ branch in $T$, or $\gamma$ is the longest common initial segment of $\alpha$ and $\beta$.

Intuitively, $\widetilde{I}(X=\alpha;Y=\beta)$ is the mutual information between $X=\alpha$ and $Y=\beta$, that is, the information of $Y=\beta$ that is contained in $X=\alpha$. In another word, it is the information required to determine the node $\gamma$ at which $\alpha$ and $\beta$ branch.

We notice that, to determine $\gamma$ is to determine $\delta$ for all $\delta$ with $\lambda\subset\delta\subseteq\gamma$. For each such a $\delta$, both $x\in T_{\delta}$ and $y\in T_{\delta}$ occur, we need to determine the codeword of $T_{\delta}$ in $T_{\delta^{-}}$, for which the amount of information required is $-\log_2\frac{{\rm vol}(\delta)}{{\rm vol}(\delta^{-})}$. So, intuitively, $\widetilde{I}(X=\alpha;Y=\beta)$ is the mutual information between $X=\alpha$ and $Y=\beta$, in terms of the codeword of $T_{\delta}$ in $T_{\delta^{-}}$.

Define

\begin{equation*}
\widetilde{I}^T(G)=\frac{1}{{\rm vol} (G)}\sum\limits_{\substack{e=(x,\ y)\\ x,\ y\in V}}\widetilde{I}(X=\alpha;Y=\beta),
\end{equation*}
where $X$ is the codeword of vertex $x$, and $Y$ is the codeword of vertex $y$, accessible from random walk from $x$.

$\widetilde{I}^T(G)$ is then the average mutual information between all the channels represented by the edges of $G$.

Our definition of $\mathcal{C}^T(G)$ in Definition \ref{eqn:compressing-tree} is actually $\widetilde{I}^T(G)$.

\begin{lemma} \label{lem:Compressing information} Let $G=(V,E)$ be a connected simple graph, and $T$ be an encoding tree of $G$. Then

\begin{equation}
\mathcal{C}^T(G)=\widetilde{I}^T(G).
\end{equation}

\end{lemma}
\begin{proof}
According to the definition of $\widetilde{I}(X=\alpha;Y=\beta)$, for every vertex $x$ and vertex $y$, for which there is an edge from $x$ to $y$, and $x$ and $y$ have codewords $\alpha$ and $\beta$ in $T$, respectively. Let $\gamma=\alpha\cap\beta$, that is, $\gamma$ is the longest initial segment of both $\alpha$ and $\beta$, then for every $\delta$, if $\lambda\subset\delta\subseteq\gamma$, then 
both $x$ and $y$ are in $T_{\delta}$. Therefore the edge from $x$ to $y$ contributes $-\log_2\frac{{\rm vol}(\delta)}{{\rm vol}(\delta^{-})}$ to $\widetilde{I}(X=\alpha;Y=\beta)$.

Note that for any $X=\alpha$, and $Y=\beta$,

$$\widetilde{I}(X=\alpha;Y=\beta)=-\sum\limits_{\substack{\delta\\ \lambda\subset\delta\subseteq\beta}}\frac{\log V_{\delta}}{V_{\delta^-}}-\widetilde{H}^T(Y=\beta| X=\alpha).$$

This ensures that

\begin{eqnarray*}
\widetilde{I}^T(G)&=&\frac{1}{{\rm vol} (G)}\sum\limits_{\substack{e=(x,\ y)\\ x,\ y\in V}}\widetilde{I}(X=\alpha;Y=\beta)\\
&=&-\frac{1}{{\rm vol}(G)}\sum\limits_{\substack{e=(x,y)\\ x, y\in V}}\sum\limits_{\substack{\delta\\ \lambda\subset\delta\subseteq\beta}}\log\frac{V_{\delta}}{V_{\delta^-}}-\frac{1}{{\rm vol}(G)}\sum\limits_{\substack{e=(x,y)\\ x,y\in V}}\widetilde{H}(Y=\beta|X=\alpha)\\
&=&-\sum\limits_{\substack{\alpha\\ \alpha\not=\lambda}}\frac{V_{\alpha}}{{\rm vol}(G)}\log\frac{V_{\alpha}}{V_{\alpha^-}}+\sum\limits_{\substack{\alpha\\ \alpha\not=\lambda}}\frac{g_{\alpha}}{{\rm vol}(G)}\log\frac{V_{\alpha}}{V_{\alpha^-}}\\
&=&-\sum\limits_{\substack{\alpha\\ \alpha\not=\lambda}}\frac{V_{\alpha}-g_{\alpha}}{{\rm vol}(G)}\log\frac{V_{\alpha}}{V_{\alpha^-}}\\
&=&\mathcal{C}^T(G),
\end{eqnarray*}
where $g_{\alpha}=|E(\bar{T_{\alpha}}, T_{\alpha})|$, that is, the number of edges from the complement of $T_{\alpha}$, i.e., $\bar{T_{\alpha}}$, to $T_{\alpha}$, ${\rm vol}(G)$ is the volume of $G$, that is, the total degree of vertices in $G$, ${\rm vol}(\beta)$ is the volume of the vertices set $T_{\beta}$, and $\alpha^{-}$ is the parent node of $\alpha$ in $T$.

\end{proof}

Lemma \ref{lem:Compressing information} shows that given an encoding tree $T$ of $G$, the compressing information of $G$ by $T$ is the average mutual information over all the possible channels presented by the edges of $G$ under the encoding given in the encoding tree $T$.

\begin{figure}
\centering

    \includegraphics[width=0.4\textwidth]{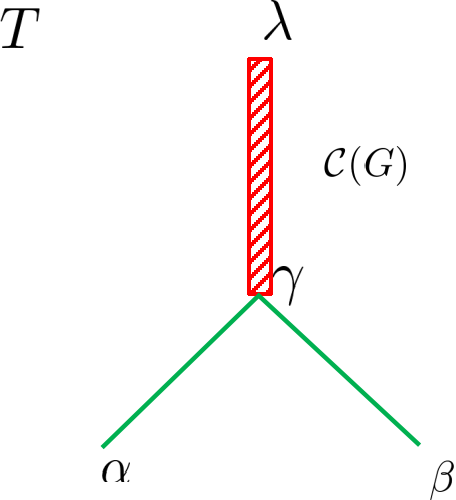}
    \caption{Compression Information. Suppose that $\alpha$ and $\beta$ are the codewords of vertices $x$ and $y$, respectively, and $(x,y)$ is an edge of $G$. We regard the edge $(x,y)$ as a channel of communications in $G$. We use $X$ to denote the codeword of vertex $x$ and $Y$ to denote the codeword of vertex $y$. Since communications in $G$ are randomly performed. Suppose that $X=\alpha$ and $Y=\beta$. Then the mutual information $I(X;Y)$ is the information gain for $Y$ when we know $X=\alpha$. In this case, $I(X;Y)$ is defined by the information required to determine the tree node $\gamma$ at which $\alpha$ and $\beta$ branch. The compressing information of $G$ given by encoding tree $T$ is the weighted average of $I(X;Y)$ for all the edges in $G$.}
    %\label{fig:Learning Model}
\end{figure}

The intuition of Lemma \ref{lem:Compressing information} is given in Figure 4 below.
In Figure 4, $\alpha$ is the codeword of the vertex from which random walk in $G$ starts, and $\beta$ is the codeword of the vertex that is accessible from random walk from $x$ in $G$. This means that $(x,y)$ is an edge in $G$. The mutual information of $\alpha$ and $\beta$ is determined by the path from the root $\lambda$ to $\gamma$, where $\gamma$ is the node in $T$ at which $\alpha$ and $\beta$ branch. The compressing information of $G$ by $T$ is the average amount of mutual information of the codewords of the two endpoints of edges for all the edges in $G$. We can also understand an edge $(x,y)$ of $G$ as a communication channel. The information needed to determine the node $\gamma$ is the information explicitly shared by the codewords $\alpha$ and $\beta$. Then the compressing information is the weighted average of the information shared by all the edges in $G$.

By using Definition \ref{def:compressing information given by an encoding tree}, we are able to define the compressing information of $G$.

\begin{definition} (Compressing information of a graph) \label{def:compressing-information} Let $G=(V,E)$ be an undirected and connected graph. We define the compressing information of $G$ as

\begin{equation}\label{eqn:CI}
\mathcal{C}(G)=\max_T\{\mathcal{C}^T(G)\},
\end{equation}
where $T$ ranges over all the encoding trees of $G$.

\end{definition}

The same as structural entropy, we can define various restricted types of compressing information of a graph.

\begin{definition} ($k$-dimensional compressing information of a graph) \label{def:compressing-information-k} Let $G=(V,E)$ be an undirected and connected graph, and $k$ be a natural number. We define the $k$-dimensional compressing information of $G$ as

\begin{equation}\label{eqn:CI-k}
\mathcal{C}^k(G)=\max_T\{\mathcal{C}^T(G)\},
\end{equation}
where $T$ ranges over all the encoding trees of $G$ of height less than or equal to $k$.
\end{definition}

\noindent Intuitively, $\mathcal{C}(G)$ is the amount of information that has been compressed by the optimum encoding tree $T$ of $G$.

Of course, for a type $\mathcal{T}$ of encoding trees, we can define the $\mathcal{T}$-type compressing information of $G$ as 

\begin{equation}\label{eqn:CI-type}
\mathcal{C}^{\mathcal{T}}(G)=\max_{T\in\mathcal{T}}\{\mathcal{C}^T(G)\}.
\end{equation}

The compression information of a graph satisfies the following:

\begin{theorem}(Graph compressing principle)\label{thm:compressing principle}
Let $G=(V,E)$ be a connected graph. Suppose that $T$ is an encoding tree of $G$. Then:

\begin{enumerate}

\item [(1)]
\begin{equation}
\mathcal{C}^T(G)=\mathcal{H}^1(G)-\mathcal{H}^{T}(G).
\end{equation}

\item [(2)] For any natural number $k\geq 2$,

\begin{equation}
\mathcal{C}^k(G)=\mathcal{H}^1(G)-\mathcal{H}^k(G).
\end{equation}

\item [(3)]
\begin{equation}
\mathcal{C}(G)=\mathcal{H}^1(G)-\mathcal{H}(G).
\end{equation}

\item [(4)]
\begin{equation}
\mathcal{C}^{\mathcal{T}}(G)=\mathcal{H}^1(G)-\mathcal{H}^{\mathcal{T}}(G).
\end{equation}

\end{enumerate}
\end{theorem}
\begin{proof} (2), (3) and (4) follow from (1) and the definition of structural entropies of a graph.

For (1). By Definition \ref{def:compressing information given by an encoding tree} ,

\begin{eqnarray*}
\mathcal{C}^T(G)&=&-\sum\limits_{\alpha\in T, \alpha\not=\lambda}\frac{V_{\alpha}-g_{\alpha}}{V_{\lambda}}\log_2\frac{V_{\alpha}}{V_{\alpha^-}}\\
&=&-\sum\limits_{\alpha\in T, \alpha\not=\lambda}\frac{V_{\alpha}}{V_{\lambda}}\log_2\frac{V_{\alpha}}{V_{\alpha^-}}+
\sum\limits_{\alpha\in T, \alpha\not=\lambda}\frac{g_{\alpha}}{V_{\lambda}}\log_2\frac{V_{\alpha}}{V_{\alpha^-}}\\
&=&\mathcal{H}^1(G)-\mathcal{H}^T(G).
\end{eqnarray*}
where the last equality follows from Theorem \ref{thm:additive function theorem}.
\end{proof}

Note that $\mathcal{H}^1(G)=H(G)$.
Theorem \ref{thm:compressing principle} shows that the compressing information of a graph $G$ is the Shannon entropy of $G$ minuses the structural entropy of $G$. This reveals the relationship between the Shannon entropy and the structural entropy.

By Theorem \ref{thm:compressing principle} and by the definition of $k$-dimensional compressing information, we have

\begin{theorem} Let $G$ be a connected undirected graph. Then:

\begin{equation}\label{eqn:k-d compression}
\mathcal{C}^2(G)\leq \mathcal{C}^3(G)\leq \cdots\leq\mathcal{C}^{k}(G)\leq\mathcal{C}(G).
\end{equation}

\end{theorem}
\begin{proof}
By definitions.
\end{proof}

In Equation (\ref{eqn:k-d compression}), it is interested to find the least $k$ such that $\mathcal{C}^{k}=\mathcal{C}(G)$.

%\subsection{$(n,k,\rhu)$-Compressible graphs}

According to Theorem \ref{thm:compressing principle}, we define the compressing ratio of a graph $G$ to be the normalised compressing information of $G$. That is,

\begin{definition}(Compressing ratio of a graph)\label{def:securityindex} Let $G$ be a connected graph.

\begin{enumerate}

\item [(1)] We define the {\it compressing ratio of graph $G$} as follows:

\begin{equation}\label{eqn:resistanceratio}
\rho (G)=\frac{\mathcal{C}(G)}{\mathcal{H}^1(G)}.
\end{equation}

\item [(2)] For every natural number $k>1$, we define the $k$-dimensional compressing ratio of $G$ as:

\begin{equation}
\rho^k(G)=\frac{\mathcal{C}^k(G)}{\mathcal{H}^1(G)}.
\end{equation}

\end{enumerate}

\end{definition}

Based on the notion of compressing ratio, we introduce the following:

\begin{definition}\label{def:compressible graph} (Compressible graph) Let $G=(V,E)$ be a connected graph of $n$ vertices and $m$ edges. Let $k>1$ be a natural number and $\rho$ be a number in $(0,1)$.
We say that $G$ is an $(n,k,\rho)$-compressible graph, if:

\begin{equation}
\rho^k (G)\geq\rho.
\end{equation}

\end{definition}

The notion of $(n,k,\rho)$-compressibility provides a new concept for classification and characterization of graphs with potential applications in a wide range of areas.

\subsection{Decoding information of a graph}\label{subsec:decoding}

Given an encoding tree $T$ of $G$, the structural entropy of $G$ given by $T$, i.e., $\mathcal{H}^T(G)$, is the quantity of information embedded in $G$ under the encoding $T$ of $G$. The structural entropy $\mathcal{H}(G)$ is {\it the intrinsic information hidden in $G$} that cannot be eliminated by any encoding tree or any lossless encoding of $G$. In another word, there is always $\mathcal{H}(G)$ amount of information hidden in $G$ that cannot be decoded by any encoding tree, or equivalently, by any lossless encoding.

For a graph $G$, the {\it orignial information} embedded in $G$ is the one-dimensional structural entropy of $G$, referred to as the Shannon entropy of $G$. Given an encoding tree $T$ as a lossless encoding of $G$, the information embedded in $G$ under the encoding given by $T$ is $\mathcal{H}^T(G)$.

Let $\mathcal{D}^T(G)=H(G)-\mathcal{H}^T(G)$. 
The quantity $\mathcal{D}^T(G)$ is the amount of uncertainty embedded in $G$ that has been eliminated by encoding tree $T$ of $G$. 
Therefore $\mathcal{D}^T(G)$ is the amount of information of $G$ that is gained from the encoding tree $T$ of $G$. This means that $\mathcal{D}^T(G)$ is the information gain from $G$ by encoding tree $T$, or equivalently, that the encoding tree $T$ eliminates $\mathcal{D}^T(G)$ amount uncertainty in $G$.

For an encoding tree $T$ of $G$, if $\mathcal{D}^T(G)$ achieves the maximum among all the encoding trees, then $T$ represents the essential structure of $G$ due to the fact that $T$ has maximally eliminated the uncertainty embedded in $G$. 
Of course, maximizing $\mathcal{D}^T(G)$ is equivalent to minimizing the structural entropy $\mathcal{H}^T(G)$. For this reason, the encoding tree $T$ of $G$ achieving $\mathcal{H}(G)$ represents the essential structure of $G$, and hence
decodes the knowledge or semantics of $G$, since it maximally eliminates the uncertainty embedded in $G$. Therefore, an encoding tree of $G$ is actually a {\it decoder} of $G$.

To measure the information gain from an encoding tree $T$ of $G$, we introduce the following:

\begin{definition} (Decoding information of a graph) \label{def:decoding-inform}
For a connected and undirected graph $G$, let $T$ be an encoding tree of $G$.
\begin{enumerate}
\item [(1)] We define the {\it decoding information} of $G$ by $T$ as
\begin{equation}
\mathcal{D}^T(G)=\mathcal{H}^1(G)-\mathcal{H}^T(G).
\end{equation}

[{\it Remark}: $\mathcal{D}^T(G)$ is the information that is gained from $G$ by $T$.]

\item [(2)] We define the {\it decoding information of $G$} as

\begin{equation}
\mathcal{D}(G)=\max_T\{\mathcal{D}^T(G)\},
\end{equation}
where $T$ ranges over all the encoding trees of $G$.

We call an encoding tree $T$ of $G$ a decoder of $G$, if

$$\mathcal{D}^T(G)=\mathcal{D}(G).$$

\item [(3)] For every $k\geq 2$, the $k$-dimensional decoding information of $G$ is defined as

\begin{equation}
\mathcal{D}^k(G)=\max_T\{\mathcal{D}^T(G)\},
\end{equation}
where $T$ ranges over all the encoding trees of $G$ with height less than or equal to $k$.

We call a $k$-dimensional encoding tree $T$ a $k$-dimensional decoder of $G$, if:

$$\mathcal{D}^T(G)=\mathcal{D}^k(G).$$

\item [(4)] Let $\mathcal{T}$ be a type of encoding trees. We define the type-$\mathcal{T}$ decoder of $G$, if

\begin{equation}
\mathcal{D}^{\mathcal{T}}(G)=\max_T\{\mathcal{D}^T(G)\},
\end{equation}
where $T$ ranges over all the type-$\mathcal{T}$ encoding trees of $G$.

We say that an encoding tree $T$ is a type-$\mathcal{T}$ decoder of $G$, if $T$ is a type-$\mathcal{T}$ encoding trre of $G$ such that

$$\mathcal{D}^T(G)=\mathcal{D}^{\mathcal{T}}(G).$$

\end{enumerate}
\end{definition}

The intuition of Definition \ref{def:decoding-inform} is as follows:

\begin{enumerate}
\item [(i)] The information in $G$ is basically the Shannon entropy of $G$, i.e., $H(G)$, or the one-dimensional structural entropy $\mathcal{H}^1(G)$.

\item [(ii)] Given an encoding tree $T$ of $G$, by using the encoding given by $T$, the remaining amount of uncertainty of $G$ is just $\mathcal{H}^T(G)$.

\item [(iii)] The $\mathcal{D}^T(G)$ in (1) of Definition \ref{def:decoding-inform} is the uncertainty embedded in $G$ that is eliminated by using the encoding tree $T$.

\item [(iv)] By (iii) above, we can interpret the encoding tree $T$ as a {\it decoder} of $G$ that finds the {\it essential structure} of $G$ by eliminating the uncertainty in the structure of $G$.

\item [(v)] Therefore, the encoding tree $T$ can be interpreted as both {\it encoder} and {\it decoder} of $G$.

\end{enumerate}

Then, we have

\begin{theorem} (Compressing and decoding principle of graphs) \label{thm:com-doc} For a connected and undirected graph $G$,

\begin{enumerate}
\item [(1)] The decoding information of $G$ is
\begin{equation}
\mathcal{D}(G)=H(G)-\mathcal{H}(G)=\mathcal{C}(G),
\end{equation}
where $\mathcal{C}(G)$ is the compressing information of $G$.
\item [(2)] For every natural number $k\geq 2$, the $k$-dimensional decoding information of $G$ is

\begin{equation}
\mathcal{D}^k(G)=H(G)-\mathcal{H}^k(G)=\mathcal{C}^k(G),
\end{equation}
where $\mathcal{C}^k(G)$ is the $k$-dimensional compressing information of $G$.

\item [(3)] Let $\mathcal{T}$ be a type of encoding trees of $G$, the type $\mathcal{T}$-decoding information of $G$ is

\begin{equation}
\mathcal{D}^{\mathcal{T}}(G)=H(G)-\mathcal{H}^{\mathcal{T}}(G)=\mathcal{C}^{\mathcal{T}}(G),
\end{equation}
where $\mathcal{C}^{\mathcal{T}}(G)$ is the type $\mathcal{T}$-compressing information of $G$.

\end{enumerate}

\end{theorem}
\begin{proof} By the definition of decoding information in Definition \ref{def:decoding-inform} and by Theorem \ref{thm:compressing principle}.
\end{proof}

Theorem \ref{thm:com-doc} ensures that the decoding information of $G$ equals the compressing information of $G$, that is, the information that can be compressed in $G$. The theorem guarantees that any information lost in the compression of a graph $G$ can be losslessly decoded by an encoding tree $T$ of $G$. This provides a fundamental principle for graph compressing and structure decoding, and provides an interpretable principle for big data analysis.

Theorem \ref{thm:com-doc} reveals the following {\it Fundamental Principles}:

\begin{enumerate}
\item An encoding tree is both an encoder and a decoder.
\item Compressing information equals decoding information.
\item Compressing never losses information.
\item Encoding trees are lossless encoders.
\item Combining the one-dimensional structural entropy and structural entropy characterizes both the compressing information and decoding information.

\item Due to the fact that the Shannon entropy for a probability distribution is a special case of the the structural entropy, the compressing/decoding principles above hold for both structured graphs and unstructured dataset.

\end{enumerate}

\subsection{Upper bounds of decoding and compressing information}\label{subsec:lowerbound}

Theorem \ref{thm:SEprinciple-New} indicates that there are graphs such as expanders for which the conductance $\Phi (G)$ is a large constant, independent of the size of the graph, that can not be significantly compressed by any encoding tree of the graphs. On the other hand, as we have shown in \cite{LP2016a}, there are many graphs $G$ whose two-dimensional structural entropy is $\mathcal{H}^2(G)=O(\log\log n)$, where $n$ is the number of vertices in the graph. For these graphs, the compressing information and the decoding information are almost the same as that of the Shannon entropy of the graphs, hence the compressing ratio is arbitrarily close to $1$ if $n$ is large enough.

When the conductance $\Phi (G)$ is small, Theorem \ref{thm:SEprinciple-New} gives only a very weak lower bound. For example, $\Phi (G)$ could be as small as $\frac{1}{\log n}$, in which case, $\Phi (G)\cdot\mathcal{H}^1(G)$ is a constant. In addition, the lower bound in Theorem \ref{thm:SEprinciple-New} may not be tight. It is interesting to find better lower bound for the structural entropy of graphs. It is even interesting to find better lower bounds for the structural entropy for different types of graphs.

\begin{theorem} (Upper bound of compressing information of a graph) \label{thm:upper bound of compressing} Let $G$ be an undirected and connected graph. Then:

\begin{enumerate}
\item [(1)] The compressing information of $G$ satisfies:

$$\mathcal{C}(G)\leq (1-\Phi (G))\cdot H(G)+\Phi (G).$$

\item [(2)] For any natural number $k\geq 2$, the $k$-dimensional compressing information of $G$ is:

$$\mathcal{C}^k(G)\leq (1-\Phi (G))\cdot H(G)+\Phi (G).$$

\item [(3)] The compressing ratio of $G$ is:

$$\rho (G)\leq 1-\Phi (G)+\frac{\Phi (G)}{H(G)}.$$

\item [(4)] For every natural number $k\geq 2$, the $k$-dimensional compressing ratio is:

$$\rho^k(G)\leq 1-\Phi (G)+\frac{\Phi (G)}{H(G)}.$$
\end{enumerate}

\end{theorem}
\begin{proof}
By Theorem \ref{thm:SEprinciple-New}.
\end{proof}

Theorem \ref{thm:upper bound of compressing} shows that the compressing information of a graph is principally upper bounded by $(1-\Phi (G))\mathcal{H}^1(G)$, where $\Phi (G)$ is the conductance of $G$. If the conductance is as large as a constant, $\alpha$ say, then the compressing information of the graph is small. According to Theorem \ref{thm:com-doc}, if the conductance of $G$ is a constant $\alpha$, the information embedded in graph $G$ that cannot be decoded is at least $\alpha\cdot \mathcal{H}^1(G)$, which is large. For example, we see the following example.

\begin{proposition} \label{pro:complete} Let $G$ be the complete graph with $n$ vertices. Then:

\begin{equation}
\Phi (G)=\frac{n}{2(n-1)}.
\end{equation}

\end{proposition}
\begin{proof}
By definition.
\end{proof}

By Proposition \ref{pro:complete}, the compressing information is at most $\frac{1}{2}\cdot (\log_2n-1)$. Therefore, there is at least $\frac{1}{2}\log_2 n$ amount of information embedded in the graph that cannot be decoded.

Therefore, Theorem \ref{thm:upper bound of compressing} shows the mathematical limitation of both compressing information and decoding information of a system. It means that there exist systems in which the information embedded cannot be significantly decoded.

\section{Structural Information Principle for Clustering and Unsupervised Learning: Decoder, Knowledge Tree, Abstraction and Tree of Abstractions}\label{sec:decoder}

Clustering or graph clustering is a classical problem in the area of unsupervised learning. It is true that clustering has provided a number of fundamental ideas for learning.
However, it has been a long-standing challenge to define a criterion for clustering. In this section, we will build a structural information theoretical principle for clustering, and for unsupervised learning, in general.

\subsection{Syntax and semantics}

A real world object is usually represented by a data point. A data point is usually a vector, whose coordinates represents the features of the data point. The features of an object divided into two classes, representing the syntax and semantics of the object, respectively.

\begin{definition} (Syntax and Semantics of an object) \label{def:syntax-semantics} Given a real world object, $o$ say, we define

\begin{enumerate}
\item [(1)] The syntax of object $o$ is the set of features that specify what $o$ is.
We use $g(o)$ to denote the syntactical features of $o$.

\item [(2)] The semantics of object $o$ is the set of features that specify what does the roles object $o$ have.
We use $h(o)$ to denote the semantical features of $o$.

\item [(3)] We define $f(o)$ to be the set of all the features of $o$, including the syntactical and semantical features, that is, the union of $g(o)$ and $h(o)$.
\end{enumerate}

\end{definition}

We notice that,  human learning has the following characters:

\begin{enumerate}
\item People learn both the syntax and the semantics of an observed object. In fact, semantics of objects probably are the more important features than the syntax for human learning.
\item Learning even if for an object could be a partial learning, meaning that learning from observing may obtain only some of the features of the object, instead of the complete set of all the features of the object.
\item Human learning may only be an approximating learning of an object, meaning that the features learnt may contain noises.

\item In the procedure of human learning, when one data is observed, we usually build the connection of the data with the data space we learnt, with the knowledge we learnt, and with the laws we learnt previously. This mechanism, referred to as connecting and/or associating
data, naturally builds the system of observed data points. We call this system as {\it data space}. This gives rise to a system of data points together with the relationships among the data points. For a data space of this form, we know that the laws of the data space are embedded in the space. This means that the data space is a mixture of laws and noises. The mission of information processing is just to distinguish the laws from the noises in the data space.

\item In the procedure of human learning, we usually learn the knowledge and laws of the data space, where the knowledge of a data space is the {\it functional modules} of the data space, and the {\it laws or rules} that generate the knowledge of the data space. Precisely, we assume that the data space must have a semantics. The semantics of data space consists of functional modules of the data space. This semantics of functional modules must be supported by a syntax of the data space. The supporting syntax of the data space is called the {\it essential structure} of the data space. The key to data analysis is to find the essential structure of a data space that supports the functional modules or semantics of the data space. This form of learning is a general procedure of human learning. Luckily, structural entropy minimization naturally follows this procedure of human learning.

\item The {\it unification of syntax and semantics} provides the {\it criterion} for the structural information learning machines.
\end{enumerate}

\subsection{Decoder: Essential structure of a graph}

Suppose that a data space or a physical system $G=(V,E)$ is given as a graph. We will need to decode the information embedded in $G$ (i.e., to find a way to eliminate the uncertainty embedded in $G$), to build the knowledge of $G$ and to discover the laws of $G$.

The decoding information in Definition \ref{def:decoding-inform} ensures that minimization of structural entropy is a natural criterion for graph clustering. 

\begin{definition} (Essential structure of a graph)\label{def:essential-structure} Let $G=(V,E)$ be an undirected connected graph.

\begin{enumerate}

\item [(1)] We call an encoding tree $T$ of $G$ an essential structure of $G$, if

$$\mathcal{H}^T(G)=\mathcal{H}(G).$$

In this case, we also call the encoding tree $T$ a decoder of $G$.

\item [(2)] For natural number $k\geq 2$, we say that an encoding tree $T$ of $G$ is a $k$-dimensional essential structure of $G$, if $T$ is an encoding tree of $G$ of height within $k$, and

$$\mathcal{H}^T(G)=\mathcal{H}^k(G).$$

In this case, we also call the encoding tree $T$ a $k$-dimensional decoder of $G$.

\item [(3)] Let $\mathcal{T}$ be a type of encoding trees, and $T$ be a type $\mathcal{T}$-encoding tree of $G$. We say that $T$ is a type $\mathcal{T}$-essential structure of $G$, if

$$\mathcal{H}^T(G)=\mathcal{H}^{\mathcal{T}}(G).$$

Here, we also call encoding tree $T$ a type $\mathcal{T}$-decoder of $G$.

\end{enumerate}

In each of the three cases (1), (2) and (3) above, if $T$ is an approximate solution of the decoder, we call it an approximate decoder of $G$.

\end{definition}

We notice that there are many ways to approximate the decoder of $G$ in all the cases such that the approximate algorithms are highly efficient with remarkable performance in quality.

According to Theorem \ref{thm:com-doc}, for a given graph $G$, the Shannon entropy or the one-dimensional structural entropy of $G$ is fixed, therefore, if $T$ is an essential structure of $G$ as defined in Definition \ref{def:essential-structure}, then the decoding information of $G$ given by $T$ is maximized. This means that the
encoding tree $T$ eliminates the maximum amount of uncertainty embedded in $G$. This implies that the essential structure, an encoding tree, ensures that the uncertainty left in $G$ is minimized. Hence, $T$ encodes the most robust syntactical structure of $G$.

We note that the essential structure of a graph may not be unique. It is possible that a graph has several essential structures. However, nevertheless, any essential structure determines a robust and stable syntax due to the fact that it has eliminated the maximum amount of  uncertainty embedded in the graph.

\subsection{Semantical interpretation principle: Knowledge tree of a physical system}\label{subsec:knowledge-tree}

Suppose that the graph $G=(V,E)$ is a physical system in the real world, and that $G$ represents the syntax of the system of many-body objects together with their relationships. We emphasize that $G$ represents the syntax of a system, and that the semantics of the system 
is represented by an associating dataset outside of $G$. The syntax and semantics of $G$ certainly are closely related.

We assume that the relationship between the syntax and the semantics of a system satisfies the following properties.

{\bf Syntax and semantics hypothesis}:

\begin{enumerate}
\item [(i)] Every object in $G$ has a semantics,  

\item [(ii)] The physical system $G$ has a semantics, 

\item [(iii)] The semantics of physical system $G$ consists of functional modules of $G$, and
\item [(iv)] The semantics of the form of functional modules of $G$ must have a supporting structure, which is the supporting syntax of the semantics of $G$.
\end{enumerate}

The semantics of the physical system $G$ must be supported by a syntax structure. This supporting syntax structure should be robust. 
The essential structure in Definition \ref{def:essential-structure} is a well-defined such structure. On the other hand, once we find the supporting syntax structure of a system, we are able to acquire the knowledge (that is, the semantics) of the system. Due to the fact that the decoder or essential structure of a system is an encoding tree, the knowledge interpreted from the encoding tree is a tree as while, which we called knowledge tree of the system. This naturally leads to the following:

\begin{definition} (Knowledge tree of a graph)\label{def:knowledge-tree} Let $G=(V,E)$ be a physical system and $T$ be the encoding tree of $G$ that minimizes the structural entropy of $G$, i.e., the essential structure, or, decoder of $G$.

At first, we assume that for a real world object $x$, it is possible there are two kinds of features, the first class is the set of features that specify what $x$ is, referred to as syntax features, and the second class is the set of features that specify the roles and functions of $x$, referred to as semantics of $x$. We use $F(x)$ to denote the set of semantics features of $x$.

Then:

\begin{enumerate}
\item [(1)] Assume that for every object $x\in V$, there is a set of functional features $F(x)$ to denote the semantics of $x$.

\item [(2)] For every node $\alpha\in T$, we define the semantics of $\alpha$, written $F(\alpha)$ to be the set of features $f$ such that for every $x\in T_{\alpha}$, $f\in F(x)$. Equally, we define the semantics of $\alpha$ by

\begin{equation}
F(\alpha)=\cap_{x\in T_{\alpha}} F(x).
\end{equation}

$F(\alpha)$ is actually the common features of all the data points in $T_{\alpha}$, which can be regarded as an {\it abstraction} of all the data points $x\in T_{\alpha}$.

[Remark: (i) The terminology ``abstraction" usually indicates a set of general features of an object. In this case, the abstraction is the set, $S$, say, of key features of the object such that there are many more objects share the same set $S$ of features extracted from the object. In our definition, we assume that, we have observed many objects that share some common features. In so doing, the set of common features is of course the abstraction for each of the many objects. Of course, the decoder $T$ of the system $G$ ensures that for each tree node $\alpha\in T$, the set $T_{\alpha}$ of data points must share remarkable common features. Therefore, our definition of abstraction is the same as our intuitive understanding of the notion of abstraction.

(ii) It is usually hard to determine the abstraction of a given object. The reason is that, a real world object usually have many features. We simply just don't know which of the features shared by many objects, or key to the representation of the object.

(iii) Our definition of abstraction is to find a module of the data space by algorithms on the syntax of the data space. Since the module include many objects, and the many objects form a module, i.e., a functional module as its semantics. This implies that the many objects of a module must share remarkable common features. 

(iv) Once the set of remarkable common features is built, we known that it is just the abstraction for each of the objects in the module.

(v) Our definition of abstraction actually gives rise to an easy algorithm to find the abstractions of the system. This solves the problem of ``abstracting" by using the relationships and connections among the data points.

(vi) Clearly, the abstractions of the functional modules form another space. The new space must be sparser, and lower dimensional in nature. However, we will not call it low-dimensional and sparse space of the abstractions. 

(vii) The structure of the abstractions is again a tree, which we call knowledge tree.

(viii) This paper also implies a tree representation of knowledge, as a way of knowledge representation.]

\item [(3)] We define the knowledge tree of $G$ given by $T$ as

\begin{equation}
KT=\{F(\alpha)\ |\ \alpha\in T\}.
\end{equation}

[Remark: The knowledge tree $KT$ of $G$ provides a high-level abstraction of the physical system organized as a tree. This means that abstractions have different levels, and are organized by a highly efficient data structure so that there are highly efficient algorithms (or even local algorithms, running in time ${\rm poly}\log n$),  finding the desired level of abstractions for intuitive reasoning.]

\end{enumerate}
\end{definition}

Furthermore, the knowledge tree $KT$ of system $G$ allows us to extract the {\it flows of the abstractions}, referred to as {\it laws or rules discovery}.

We have defined the knowledge tree by using the semantics of data points and the decoder or essential structure of a physical system $G$. 
Of course, we could also define the abstractions of the decoder or essential structure of $G$ by using the syntactical features or general features of data points. This will establish different types of abstracting, referred to Subsection \ref{subsec:tree-of-abstractions}.

\subsection{Laws or rules discovery}\label{subsec:rules of abstraction}

Given a physical system $G=(V,E)$, the structural entropy minimization principle determines and decodes an encoding tree $T$ of $G$ that gains the maximum amount information embedded in $G$. Such an encoding tree $T$ is hence an essential syntax structure, or for simplicity, an essential structure, of $G$. This encoding tree $T$ certainly supports a knowledge tree $KT$ of $G$. 

We introduce the following definition of rules of abstraction by using the knowledge tree $KT$ of $G$.

\begin{lemma} (Rules of abstraction)\label{lem:Abstraction} Given a physical system $G=(V,E)$, suppose that $T$ is a decoder of $G$. 
For a fixed data point $x\in V$, let $\alpha$ be the codeword of $x$ in the decoder $T$. Suppose that 
$$\alpha=\alpha_l, \alpha_{l-1},\cdots, \alpha_0=\lambda$$
is the path from $\alpha$ to the root node $\lambda$ in $T$.

For every $j=l, l-1,\cdots, 1,0$, let $F_j=F(\alpha_j)$. Then:

\begin{equation}\label{eqn:abstraction}
F_0\subseteq F_1\subseteq\cdots\subseteq F_l.
\end{equation}

\end{lemma}
\begin{proof}
By Definition \ref{def:knowledge-tree}.
\end{proof}

Usually, the inclusions in Equation (\ref{eqn:abstraction}) are proper, in which case, for every $j$, $|F_j|<|F_{j+1}|$. 

Note that $\alpha=\alpha_l$ is the codeword of $x$ in the decoder $T$. Hence $F(\alpha)$ is the set of features of data point $x\in V$.

According to Lemma \ref{lem:Abstraction}, we introduce the following:

\begin{definition} (Rules Extracting)\label{def:rules-extracting} Assume the notions in Lemma \ref{lem:Abstraction}. For every $j$ with $l>j\geq 0$, we call $F_j$ the $j$-th level abstraction of data point $x$, written $F_j(x)$.
\end{definition}

\begin{definition} (The $j$-th level abstraction)\label{def:j-th level abstraction} Assume the notations in Definition \ref{def:rules-extracting}.
For a fixed $j$, we define the $j$-level abstraction of physical system $G$ to be the class

\begin{equation}
F_j(G)=\{F_j(x)\ |\ x\in V\}.
\end{equation}

\end{definition}

Clearly, for small $j$, $F_j(G)$ can be represented in a low-dimensional space. This means that high-level abstractions of a system can be realized in low-dimensional space, reflecting the essence of data abstraction.

\subsection{The laws of a physical system}

In Subsection \ref{subsec:rules of abstraction}, we have defined the high-level abstractions of the decoder and the knowledge tree of a physical system. 

According to Lemma \ref{lem:Abstraction}, for every data point $x\in V$, the high-level abstractions of $x$ satisfies

\begin{equation}\label{eqn:flow-1}
F_0(x)\subseteq F_1(x)\subseteq\cdots \subseteq F_l(x).
\end{equation}

\begin{definition} (Flow of abstractions)\label{def:flow of abstraction} 

\begin{enumerate}
 \item [(1)] For a data point $x\in V$, assume the notations in Definition \ref{def:rules-extracting}, we define the flow of abstractions of $x$ to be the following sequence:
 \begin{equation}
 {\rm flow}(x)=:\{F_l(x)\supseteq\cdots\supseteq F_1(x)\supseteq F_0(x)\}.
 \end{equation}
\item [(2)] We define the flows of abstractions of $G$ by

\begin{equation}
{\rm Flow}(G)=\{{\rm flow}(x)\ |\ x\in V\}.
\end{equation}

\end{enumerate}
\end{definition}

\begin{proposition}(Flow of abstractions proposition)\label{prop:flow-of-abstractions} The flows of abstractions ${\rm Flow}(G)$ form a tree.
\end{proposition}

\begin{proof} By definition.
\end{proof}

We use $T_f(G)$ to denote the tree of flows of abstractions, that is, $T_f(G)$ is exactly ${\rm Flow}(G)$. From the tree of flows, it is easy to find, for any given two vertices $x$ and $y$, the {\it least common abstractions} of $x$ and $y$ will be the set of features located at the node in $T_f(G)$ at which $x$ and $y$ branch.

\begin{definition} (Laws of system $G$) We define the laws of $G$ to be the rules that generate the flow of abstractions of $G$, that is, the rules of ${\rm Flow}(G)$, or the tree $T_f(G)$.
\end{definition}

According to the tree $T_f(G)$, we are able to find the least common abstractions of arbitrarily given sequences $x_1, x_2,\cdots, x_l$ of vertices $x_j\in V$, for $j=1,2,\cdots,l$. This property allows us to operate on the abstractions of objects.

\subsection{Tree of abstractions}\label{subsec:tree-of-abstractions}

Given a graph $G$, we assume that $G$ represents the syntax of a physical system consisting of many bodies together with their relationships, and that every object of the system has a semantics that are associated, but outside of the system. We notice that 
the system $G$ is determined by the syntax of the many bodies, although each of the many bodies has an associated semantics.

Suppose that $T$ is a decoder of $G$ found by the structural entropy minimization principle. Syntactically speaking, $T$ is the encoding tree of $G$ such that $T$ has gained the maximum amount of information embedded in the system $G$, and that the structural entropy of $G$ given by $T$, i.e., $\mathcal{H}^T(G)$ has been already the intrinsic information embedded in $G$ that cannot be decoded by any encoding tree or any lossless encoding of $G$. Due to this feature of $T$, $T$ certainly determines a semantical interpretation of system $G$. 

The knowledge tree in Definition \ref{def:knowledge-tree} has defined a function $F(\alpha)$ associated with tree node $\alpha\in T$ such that for each tree node $\alpha\in T$, $F(\alpha)$ is the set of common features of all the objects in $T_{\alpha}$. In so doing, $F(\alpha)$ is actually an {\it abstraction} for each of the object in $T_{\alpha}$.

By the definition of $T$ and $F$, we have that for any tree nodes $\alpha, \beta$, if $\alpha\subset\beta$, then $F(\alpha)\subseteq F(\beta)$.

We will define the tree of abstractions $T$ such that every node $\alpha\in T$ is associated with a set $F(\alpha)$ of features and such that for any $\alpha,\beta\in T$, if $\alpha\subset\beta$, then $F(\alpha)\subset F(\beta)$.

\begin{definition} (Tree of abstractions)\label{def:tree-of-abstractions} Let $T$ be a decoder of $G$ and $F$ be the associated abstractions. We define the tree of abstractions to be the tree $T^{*}$ obtained from $T$ by the following operations:

For any $\alpha\subset\beta$, if $F(\alpha)=F(\beta)$, then merge $\beta$ to $\alpha$.

Then the tree of abstractions is the pair $\langle T^{*}, F\rangle$ constructed as above.

\end{definition}

\begin{proposition} (Tree of abstractions proposition)\label{prop:tree of abstractions}

For a tree of abstractions $\langle T^{*}, F\rangle$, the following property holds: For any nodes $\alpha,\beta\in T^{*}$, if $\alpha\subset\beta$, then

\begin{equation}
F(\alpha)\subset F(\beta).
\end{equation}

\end{proposition}

\begin{proof}
By definition and construction of $T^*$.
\end{proof}

Of course, the definition of tree of abstractions may not be unique, because, the decoder may not be unique.

The concept of tree of abstractions provides us a structure and representation of abstractions. The tree structure of abstractions naturally captures the intuition of human abstraction, and the abstractions in the intuitive reasoning of humans. More importantly, the tree structure of abstractions allows highly efficient algorithms for finding the abstractions desired and for operating on the abstractions.

The concept of tree of abstractions is  perhaps fundamental to understand human reasoning and natural language communications.

The concept of trees of abstractions allows us to explicitly define and represent the high-level abstractions of physical systems or knowledges. The tree representation of abstractions supports highly efficient algorithms to perform reasoning in different-level of abstractions. The idea of trees of abstractions shows that abstract concepts such as ``abstracting" can be explicitly defined and represented, and that such representations support machines to perform intuitive reasoning.

The trees of abstractions using semantics, syntax, and general features allow us to establish the relationships between syntax and semantics of systems, realizing the unification of syntax and semantics.

\section{Decoding Information Maximization Principle: Connecting Data}\label{sec:connecting-data}

In Section \ref{sec:decoder}, we proposed the model of learning from a decoder, or essential structure of a physical system, provided that the physical system has been built and given. However, in practice, we don't have a structured data space. 
A grand challenge: How to build the system of data points? What is the principle for us to build the data space?

%\subsection{Constructing the system of dataset}

Suppose that we have observed a set $V=\{x_1, x_2,\cdots, x_n\}$ of data points $x_i$and that the observation for each data point $x=x_i$ includes both a syntactic and a semantic set of features as in human learning.

In practice, there are many ways to define the relationships for any pair $(x_i,x_j)$ for data points $x_i$ and $x_j$. However, we don't know which way is the best for us to construct the data space.

Usually, there is a parameter $\theta$, say, which determines the ways of structuring of the unstructured data points $V$. Let $G_{\theta}$ be the graph constructed from parameter $\theta$. 

According to Definition \ref{def:decoding-inform}, our principle for linking the data points is to find the $\theta$, $\theta_0$, say, satisfying:

\begin{equation}\label{eqn:structuring}
\theta_0=\arg\max_{\theta}\{\mathcal{D}(G_{\theta})\}.
\end{equation}

Then let $G=G_{\theta_0}$. By the choice of $\theta_0$, $G$ allows us to gain the maximum {\it decodable information} from dataset $V$, that is, using $\theta_0$, we are able to gain the maximum amount of information from the dataset. 

Therefore, decoding information maximization is the principle for structuring an unstructured dataset. 

Equation (\ref{eqn:structuring}) provides the principle for structuring unstructured dataset. We call this principle as {\it decoding information maximization principle}. This ensures that, the way we create links allow us to gain the maximum amount of information. This also means that we have eliminated a maximum amount of uncertainty in the dataset, or we have learnt the maximum knowledge from the dataset.

\section{Decoding Information Maximization Principle: Learning from Observing, Connecting and Associating}\label{sec:observing-connecting-associating}

In Section \ref{sec:decoder}, we proposed the model of learning from a decoder, or essential structure of a physical system, provided that the physical system has been built and given.
In Section \ref{sec:connecting-data}, we introduced the principle for connecting data points to build a data space. After the data space is built, then we may decode the data space using the method in Section \ref{sec:decoder} to learn the knowledge and laws of the data space.

 However, in practice, human learns dynamically. In the procedure of human learning, when one observes a new data, he/she first tries to find the relationship between the newly observed data and the data space he/she observed previously. In particular, when one observes a new data, he/she may have many evidences or even imagination to build the connection between the newly observed data and the data space he/she has built before. This is of course a fundamental mechanism of human learning.

 Clearly, different methods of building connections between newly observed data points to existing data space may directly determine the performance of learning.

In the procedure of human learning, when we observed a data, we immediately link the data to the data space, use the knowledge and laws we have learnt previously to reason, and update our knowledge and laws by using the new observation and our old knowledge. This means that building the data space is the result of human learning, and more importantly, human links the newly observed data to the existing data space by even associating, which would be something from imagination.

 Intuitively, the connections between the newly observed and the existing data space should not be too many, and not too less. More importantly, the connections between newly observed data points and existing data space should be helpful for us to form new knowledge and to extract new laws.

  What is the mathematical principle to realize the intuition above?

To better understand the role of linking data or associating in human learning, let us look at the procedure of human observing. A human {\it learner} proceeds as follows:

\begin{enumerate}

\item Observing  a data point $x$

%In this step, different people may have different abilities in observing a real world object. 

\item Building the connections between $x$ and the data space $S$ built previously

\end{enumerate}

%This leads to a grand challenge: How to build the system of data points? What is the principle for us to build the data space? 

%\subsection{Constructing the system of dataset}

%Suppose that we have observed a set $V=\{x_1, x_2,\cdots, x_n\}$ of data points $x_i$. The observation for each data point $x=x_i$ includes both a syntactic and a semantic set of features. In practice, there are many ways to define the relationships for any pair $(x_i,x_j)$ for data points $x_i$ and $x_j$. However, we don't know which way is the best for us to construct the data space.

Suppose that there is a parameter $\theta$ that determines the connections between newly observed data points $X$ and existing data space $S$. Let $G_{\theta}$ be the graph constructed from $S$ by connecting $X$ to $S$ using parameter $\theta$. 

According to Definition \ref{def:decoding-inform}, our principle for linking the data points $X$ to data space $S$ is to find the $\theta$, $\theta_0$, say, satisfying:

\begin{equation}\label{eqn:associating}
\theta_0=\arg\max_{\theta}\{\mathcal{D}(G_{\theta})\}.
\end{equation}

Then let $G=G_{\theta_0}$. By the choice of $\theta_0$, $G$ allows us to gain the maximum decodable information from the connections between $X$ and data space $S$.

%dataset $V$, that is, using $\theta_0$, we are able to gain the maximum amount of information from the dataset. 

Therefore, decoding information maximization is the principle for associating data points to existing data space.

%structuring an unstructured dataset. 

Equation (\ref{eqn:associating}) provides the principle for linking data and for associating in learning. We call this principle as {\it decoding information maximization principle}. This ensures that, the way we create links between new data points and existing data space allow us to gain the maximum amount of information. This also means that we have eliminated a maximum amount of uncertainty, or we have learnt the maximum knowledge.

\section{Merging of Logical Reasoning and Intuitive Reasoning: Tree Method of Reasoning}\label{sec:logical-intuitive-reasoning}

\subsection{General method of logical and intuitive reasoning}

During the procedure of human learning, a person is constructing decoders $T_1, T_2,\cdots, T_N$, knowledge trees $K_1, K_2,\cdots, K_N$ and trees of abstractions $A_1, A_2,\cdots, A_N$ of data spaces $S_1, S_2,\cdots, S_N$. In the process of learning, one may observe a new data point, $x$ say. During the observation of $x$, one may observe a set of syntactical features $G(x)$, and a set of semantical features $F(x)$, of $x$, respectively. Due to the fact that $N$ is small, and the trees of abstractions $A_1, A_2,\cdots, A_N$ are highly sparse. It is easy to identify an $i$ and the longest node $\alpha$ in tree $A_i$ of abstractions such that the syntactical abstraction associated with $\alpha$, denoted by $G(\alpha)$ is a subset of the set of the syntactical features of $x$, that is, $G(\alpha)\subseteq G(x)$. By the choice of $\alpha$, the data point $x$ should be encoded in $T_i(\alpha)$, that is, module $\alpha$ of decoder $T_i$. This step is an intuitive reasoning by using trees of abstractions. After deciding the module $\alpha$ of decoder $T_i$, we may create links from $x$ to the data points in $T_i(\alpha)$ by using the information gain maximization principle to update data space $S_i$, and then update the decoder $T_i$ of data space $S_i$. The latter steps correspond to a local reasoning of computation step by step, in a local area of data spaces. 

The procedure above is a learning from observing that is highly alike human learning. We will see that this procedure of learning is the most remarkable character of our structural information learning machines. This is perhaps the most fundamental advantage of our structural information learning machines.

To better understand the procedure, we introduce the detailed steps of the learning procedure using a simple example. In the real world applications, one may have sever decoders, knowledge trees, and trees of abstractions corresponding to the learning of different subjects.
Of course, the different trees may be further organized as a general tree such that each of the previous trees as a subtree of the general tree. 

[Remark: It would be highly likely that human learning is just to build some knowledge trees and some trees of abstractions.]

Learning is a dynamically evolving procedure step by step. It proceeds as follows.

Suppose that at the end of time step $t$, we have observed a data space $G=(V,E)$, learnt an essential structure $T$, i.e., a decoder (which is also the encoder) of $G$, a knowledge tree $KT$ of $G$.

At time step $t+1$:

\begin{enumerate}
\item [(1)] (Observing data) Let $x$ be a new data observed at step $t+1$. 

When we observe data point $x$, we obtain simultaneously the syntactical features and semantical features of $x$. 
Let $I(x)$ be the set of syntactical features of $x$, $J(x)$ be the set of semantical features of $x$, and $F(x)$ be the set of all the features of $x$.

[Remark: After observing a data point, it is important to build the connections between the newly observed data point to the data points observed previously. The connections may have two classes, the first class is global connections, and the second class is local connections.]

\item [(2)] (Associating) Let $\theta$ be a parameter that determines the ways to link $x$ to some data points in $G$. 
Let $\theta_0$ be the parameter according to the principle in Equation (\ref{eqn:associating}). 
Then set $G+\{x\}$ to be the graph obtained from $G$ by adding vertex $x$ with ways of edges determined by parameter $\theta_0$.

Whenever we observed a data point, we may have a general method to build the connections between the newly observed data point and the existing data space. For this, we first use the information gain maximization to build the connections.

\item [(3)] (Choosing abstraction) Let $A$ be the tree of abstractions defined from the decoder using either syntactical features or all the features of data points.

Let $\alpha$ be the longest $\beta$ in tree $A$ of abstractions such that $A(\beta)\subseteq I(x)$, where $A(\beta)$ is the set of features associated with node $\beta\in A$.

\item [(4)] (Linking data) We now know that $A(\alpha)$ is the abstraction of $x$. It is possible that $x$ may have some special connections with the data points in $T_{\alpha}$, where $T$ is the decoder.

Suppose that there is another parameter $\theta$ that determines the connections between $x$ and data points in $T_{\alpha}$. The principle for choosing parameter $\theta$ is the maximization of information gain in Equation (\ref{eqn:associating}).  Using the principle, we build some special connections between $x$ and $T_{\alpha}$.

Now we have built the data space with $x$ inserted, written $G+\{x\}$, and we know that $x$ should be encoded in the module of $T_{\alpha}$.

\item [(5)] (Intuitive reasoning) Set

$$T'_{\alpha}\leftarrow T_{\alpha}\cup \{x\}.$$

Step (5) performs the encoding of $x$ into $T_{\alpha}$. However, we have not decided the exact codeword of $x$ in an updated encoding tree $T'$.

\item [(6)] (Locally logical reasoning) Update the branch with top node $\alpha$ to minimizing the structural entropy of the graph $G+x$, and obtain an encoding tree $T'$ of $G+x$. $T'$ is an approximation of the decoder of $G$ together with a newly observed data $x$.

Once we have decided to enumerate $x$ into $T_{\alpha}$, we will need to further determine the codeword of $x$ in a new encoding tree. 
This will be determined by some greedy strategies in the branch with root $\alpha$, which is a local computation, and hence a logical reasoning.

\end{enumerate}

The steps above described a general method of linking data and of performing both intuitive reasoning and logical reasoning. The method is  developed better fitting problems for discrete objects.

\subsection{Using abstractions in classifications}

For continuous objects, the general method above may have some variations. For this, we look at an example.

Figure 5 is the three-dimensional gene map of lymphomas found by using the three-dimensional structural entropy minimization principle, referred to \cite{LP2016a, LYP2016}.

The author and his coauthors developed a method based on structural entropy minimization principle to identify the type and subtypes of 5 tumors \cite{LYP2016}. Figure 5 is one of the tumors, that is, the lymphomas. 

In Figure 5, the horizontal line represents the cell samples of lymphomas, and the vertical line represents the genes. The color in the heat-map represents the gene expression profiles, for which the deeper the color, the higher the expression profiles.

In tumor type/subtype identification, we are given a number of cell samples, $c_1, c_2,\cdots, c_N$, and a list of genes, $g_1, g_2,\cdots, g_n$, where $n>> N$. For every $i=1,2\cdots, n$ and every $j=1,2,\cdots, N$, gene $g_i$ has an expression profile in cell $c_j$. 

Suppose that all the genes are ordered as they are listed below:

$$g_1, g_2,\cdots, g_n.$$

Then for every $j=1,2,\cdots, N$, cell sample $c_j$ is represented by a vector $v_j=(a_1, a_2,\cdots, a_n)$, where $a_i$ is the expression profile of gene $g_i$ in cell sample $c_j$. For two cell samples $c_j$ and $c_{j'}$, if there are some genes that have high expression profiles in both $c_j$ and $c_{j'}$ simultaneously, then the two cell samples $c_j$ and $c_{j'}$ are closely related. Using this intuition, we can define a weight $w$ between $c_j$ and $c_{j'}$ by using the two corresponding vectors $v_j$ and $v_{j'}$. 

Clearly this gives rise to a complete, but weighted graph, $H$ say. It is easy to see that, some weights in $H$ must be significant, but many mores are simply noises or trivial weights. Although each of the noises or trivial weights in $H$ maybe small, the collection of all the noises and trivial weights cause a big noise. Therefore, the first step of analysis is to construct a graph $G$ obtained from $H$ by keeping the significant weights, and by removing the small noises and trivial weights. 

In \cite{LYP2016}, the authors introduced some a method using the one-dimensional structural entropy. This constructs a graph $G$. By using the two- and three-dimensional structural entropy minimization algorithms, we can identify the types and subtypes of tumors. Figure 5 is the result of three-dimensional structural entropy minimization for lymphomas cell sample classification. 

In Figure 5, our algorithm identifies the lymphomas cell samples into types and subtypes such that each subtype has a set of genes such that the set of genes highly express the subtype, and express the subtype only. This means that each subtype found by our algorithm is defined by a set of genes.

In this example, we view each cell sample as a data point, and the genes are features. According to Figure 5, each type or subtype found by our algorithm has a set of genes that defines the type or subtype. For every type or subtype $X$ found by the algorithm, see Figure 5, let $Y$ be the corresponding set of genes that defines $X$. In this case, we say that $Y$ is the set of abstractions of cell samples in $X$.

Suppose that $X_1, X_2,\cdots, X_k$ are all the subtypes, and that $Y_1, Y_2,\cdots,Y_k$ are sets of abstractions of the subtypes $X_1,X_2,\cdots, X_k$, respectively.

With this definition of abstractions, whenever we observed a new lymphomas cell sample, $c$ say, we may simply compute the average 
expression profiles of the sets of abstractions $Y_1,Y_2,\cdots, Y_k$ on $c$. 
For each $j=1,2,\cdots,k$, let $b_j$ be the average expression profiles of $Y_j$ on $c$. 

Let $j$ be such that $b_j=\max\{b_1,b_2,\cdots,b_k\}$. Then we can encode cell sample $c$ into subtype $X_j$.

This is an example of using abstractions to classify a newly observed data point.

In this case, we used the average expression profiles of sets of abstractions to classify.

\begin{figure}
\centering

    \includegraphics[width=0.8\textwidth]{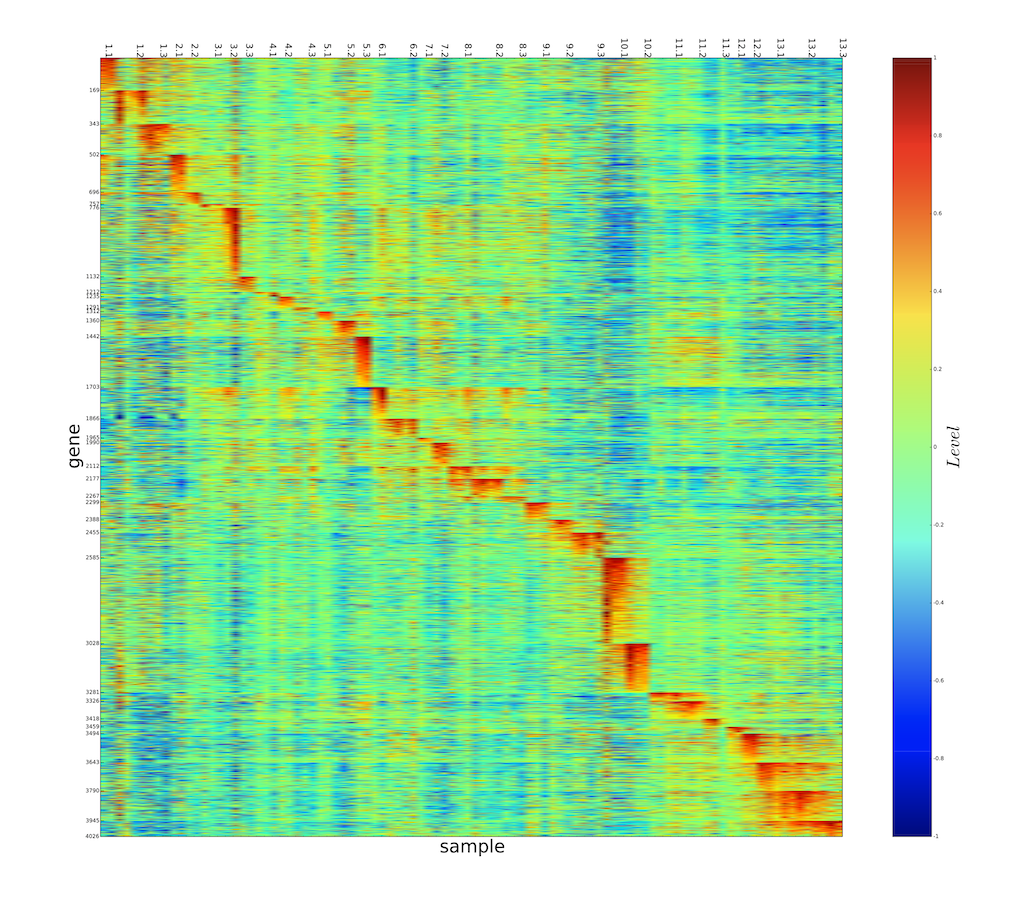}
    \caption{Lymphomas Three-Dimensional Gene Map}
    %\label{fig:Learning Model}
\end{figure}

\subsection{Constructing data space}

It is easy to see, using our structural information learning machines, the above tumor classification problem may be remarkably solved. We notice that our previous work has already been verified by clinical datasets.

New approach to solving the problem above is proposed as an example. Go back to the construction of graph $G$ of cell samples.
After we have defined a complete weighted graph $H$. We construct a graph $G$ as follows:

\begin{enumerate}
\item Let $k$ be a parameter.
\item Let $G_k$ be the graph obtained from $H$ be keeping the edges of the $k$ greatest weighs.
\item Let

\begin{equation*}
k_0=\arg\max_k\{\mathcal{D}(G_k)\}.
\end{equation*} 

Of course $\mathcal{D}(G_k)$ is approximately computed. This does not harm our results.

The choice of $k_0$ ensures that from $G_{k_0}$, we may gain the maximum decodable information.

\item Let $G=G_{k_0}$.
\end{enumerate}

This gives a new construction of cell sample graph.

\subsection{Linking data point to data space}

We will introduce the method of linking data to data space by using the tumor classification problem as an example.

Suppose that $G$ is a cell sample graph as above. When we observed a new cell sample $c$, we build the links between $c$ and cell samples in $G$ as follows.

\begin{enumerate}
\item For every cell sample $c_i$ in $G$, define $w_i$ to be the weight between $c$ and $c_i$.

\item Let $G_k$ be the graph obtained from $G$ by adding $c$ together with the edges of the top $k$ weights between $c$ and cell samples in $G$.

\item Let

\begin{equation*}
k_0=\arg\max_k\{\mathcal{D}(G_k)\}.
\end{equation*} 

%Of course $\mathcal{D}(G_k)$ is approximately computed. This does not harm our results.

%The choice of $k_0$ ensures that from $G_{k_0}$, we may gain the maximum information.

\item Let $G=G_{k_0}$.

\end{enumerate}

This constructed the graph by linking the newly observed cell sample $c$ to the existing data space $G$.

After we have defined the new graph $G$, we can easily find the types and subtypes from $G$ by using the two- and/or three- dimensional structural entropy minimization principle.

We emphasize that our method solving the problem of identification of tumor type or subtype and classification problem of cancers is completely free of hand-made parameter choice. It is principle-based. More importantly, the results of our method are highly consistent with clinical datasets.

\subsection{Learning single cell biology}

Single cell biology has become a grand challenge in both biology and medical sciences. Due to the fact that the Hi-C data points are in a single cell too sparse, it is hard to predict the topologically associating domains (TAD) in a single cell. Based on the structural information theory, the author and his coauthors \cite{L+2018} developed the method of deDoC. The method deDoC discovered for the first time TAD-like domains in 10 single cells. However, the method cannot find significant result from the Hi-C data of a single cell.

By applying the structural information learning machines, we will have a new chance to predict TAD or TAD-like structures from Hi-C data of a single cell. The new approach is as follows.

Given the Hi-C data of a gene whose loci are represented by a sequence $1, 2,\cdots,n$. For a small pair $(i,j)$, the Hi-C data set gives an interaction profiles $m_{i,j}$. We notice that $n$ is very large and the number of pairs $(i,j)$ with $m_{i,j}$ defined is small. This is the reason why we cannot predict the folded structure of the gene from the Hi-C data.

We also note that a loci is actually an interval, instead of a mathematical point. Choosing appropriate binsize $k$ may make sure that the Hi-C data of a single cell is significant in predicting the TAD-like structure in the single cell. The decoding information maximization provides such a method to choose an appropriate binsize $k$. The method proceeds as follows:

\begin{enumerate}

\item For every $k$, let $G_k$ be the graph defined by the Hi-C data of a single cell with binsize $k$.

\item Let 

\begin{equation*}
k_0=\arg\max_k\{\mathcal{D}(G_k)\}.
\end{equation*}

The choice of $k_0$ ensures that we may gain the maximum decodable information from the choice of binsize $k_0$.

\item Let $G=G_{k_0}$

\item By using the deDoc in graph $G$, we may find TAD-like structures in a single cell.

\end{enumerate}

This provides an appealing method to find TAD-like structures from the Hi-C data of a single cell.

\subsection{Learning from small dataset}

The current success of deep learning heavily depends on the availability of massive labelled dataset. However, human learning can always gain something even if from a few data points. For example, even a baby can learn something from what he/she sees.

Our structural information learning machines can learn from even a single data point. The machines do not need massive useful dataset. 

Our model shows that learning from a small dataset is possible.

\subsection{Learning from (accidental) events: Tree method of abstractions}\label{subsec:learning-from-events}

In human history, sometimes, accidents changed the world. The two World Wars started by some seemly accidents. In a society, an unexpected social event sometimes caused a social problem. In nature, some events of small probability, once occurred, may even change 
the nature. 

In human society or nature, people may observe that in certain circumstance, evens of small probability, once occurred, may cause a big change of the society and nature.
What is the principle behind this phenomena? How can we identify these kinds of small probability events?

The problem is a learning problem, learning from observations? It is a problem about the relationship between events and massive dataset.

Intuitively, the reason why some or a single event may change a massive dataset could be as follows:

\begin{enumerate}
\item [(1)] The data points in a massive dataset share some key features, $f_1, f_2,\cdots, f_k$ say, such that if these features are activated, then the massive dataset will change with high probability.
\item [(2)] The small probability event, $E$ say, once occurred, triggers the conditions that the features $f_1, f_2,\cdots, f_k$ are activated.
In this case, the small probability event $E$ plays the role of a variated gene in cancers, for example.

\end{enumerate}

The procedures (1) and (2) above are very similar to the outbreak of cancers or spreading diseases.

Therefore, in nature, society, economics, politics, science, technology and medical sciences etc, it is possible some small probability event causes a big change. Early identification of such small probability events is the key to preventing distresses from happening.

According to the arguments above, learning from small probability events consists of the following

{\bf Abstracting steps}:

\begin{enumerate}
\item (Learning an unexpected event $E$) Suppose that $E$ is an unexpected event. 

\begin{enumerate}
\item Extract as many as possible the features of $E$.
\item Analyze the relationships among the features of $E$.
\item Discover the set of key features of $E$, denoted by $K$ such that $K$ is probably shared by a massive dataset $S$.

Let $S$ be the dataset of all the possible candidate data points that share features $K$.
\item Extract the intrinsic law $L$ of the set of key features $K$.

\end{enumerate}
\item Construct the data space consisting of dataset $S$.

\item Find a tree $A$ of abstractions of data space $S$.

\item Find the longest $\alpha\in A$ such that $A(\alpha)\subseteq K$. 

This tells us, event $E$ may triggers the change for the data points with abstractions $A(\alpha)$.

\item If the laws $L$ of the key features $K$ shared by the dataset associated with $A(\alpha)$, then $A(\alpha)$ may change.

\end{enumerate}

The abstracting steps above is similar to the following intuitive procedure for predicting cancers:

\begin{enumerate}
\item [(i)] Suppose that we have found a cancer cell $C$
\item [(ii)] Suppose that $Y$ is a small set of genes each of which has high expression profile in $C$.
\item [(iii)] We hence suspect that $Y$ is the triggering condition of a type of cancer.
\item [(iv)] We will need to take measure for all the people whose cell samples have high expression profiles in each of the genes in $Y$. 

\end{enumerate}

Our structural information learning machines provide an approach to learning from small probability events as above. Because the essence of the problem is to find a tree of abstractions for a massive dataset, and to extract the key features of an event. Both phases can be learnt by structural information learning machines.

\subsection{Understanding the semantics of natural language}

In Subsection \ref{subsec:learning-from-events}, we have described a framework for learning from events. The key idea of learning from events is the tree of abstractions of the massive dataset that are probably link to the event.

This idea of learning from events may be developed to a learning machine that understands the semantics of natural languages.

A sentence, or a paragraph, or a document of natural language certainly consists of many words each with syntax, semantics and noises together with the order of words. The syntax, semantics of words must have different levels of abstractions. More importantly, the exact semantics of a sentence, or a paragraph or a document can only be better understood in an environment consisting of knowledge, culture, history, background etc. In this sense, the understanding of a sentence, or a paragraph or a document is similar to a learning from an event in  Subsection \ref{subsec:learning-from-events}.

However, the learning depends on the development of some knowledge trees and trees of abstractions of natural languages. Our SiLeM model provides the necessary mechanisms to build such machines.

[Remark: We notice that our SiLeM machines may understand natural languages. However, they do not provide mechanisms for communications in natural language. The reason is that, the machines do not generate answers to their understanding. This argument shows that SiLeM machines solve the problem by learning the laws, but the machines cannot create and design answers. It is for this reason, the author believes that except for the information theoretical definition and model of learning, we will need a new theory, ``structural game theory", to solve the problems of ``creating" and ``designing" in the generation of ``artificial intelligence".]

\section{Mathematical Model of Learning: SiLeM}\label{sec:SiLeM}

The structural information learning machinery (SiLeM, for short) is depicted in Figure 5 below

The key optimization of the learning model consists of three phases:

\begin{enumerate}
\item []Phase 1: {\it The construction of a data space}.

The principle for phase 1 is to construct a data space such that the decoding information of the data space is maximized. This ensures that we are able to gain the maximum amount of information from the constructed data space.

\item []Phase 2: {\it The finding of the decoder, that is, the essential structure of the data space}.

The principle for phase 2 is to find the encoding tree under which the structural entropy of the data space is minimized. From the found encoding tree, we have already gained the maximum amount of information embedded in the data space.

\item []Phase 3: {\it The updating of the decoder when new data is observed by using the knowledge and laws learnt previously}.

The principle for linking a newly observed data point to the existing data space is {\it to maximize the decoding information} of the updated data space.

However, applying the tree of abstractions allows us to realize the goal by {\it local and dynamical algorithms}.

\end{enumerate}

In each of the three phases, the principle for the optimization is {\it to maximize the information gain}. This gives rise to an information theoretical definition, or precisely, structural information theoretical definition, of learning.

\subsection{Mathematical principle of learning}

\begin{definition} (Learning Machine) A learning machine, written $\mathcal{L}$, is a system, satisfying the following properties:

\begin{enumerate}
\item [(1)]  The mathematical essence of a learning machine is to gain information.
\item [(2)] To gain information is to eliminate the uncertainty embedded in a system. 
\item [(3)] Eliminating uncertainty can be reduced to optimization problems, that is, information optimization problems.
\item [(4)] An information optimization problem is to find an encoding tree, or a decoder, or an essential structure of the data space with which the decoding information of the data space has been maximized.
\end{enumerate}

\end{definition}

Our definition does not distinct learning as supervised or unsupervised. Because, either supervised or unsupervised learning is to gain information, that is, to eliminate uncertainty. Traditional supervised learning may gain information by asking questions from a supervisor, and unsupervised learning has no such supervisor to ask. According to our definition, either supervised or unsupervised learning is to gain information by eliminating uncertainty. More importantly, according to our definition, learning is universal, in the sense that, whatever the subject it learns, from mathematical functions, physical objects, biology to social sciences, every learning procedure is to gain information, or equivalently, to eliminate uncertainty.

\subsection{The mechanisms of SiLeM}

The {\it mechanisms} of the structural information learning machinery (SiLeM) are:

\begin{enumerate}
\item {\it Observing data from real world}

Observing is of course the first step of learning.

\item {\it Linking data}
To construct a data space such that decoding information of the data space is maximized.

\item {\it Decoding by optimizing the information gain}

Given the data space constructed, find an encoding tree by minimizing the structural entropy of the data space.

\item {\it Interpreting the decoder to form a knowledge tree}

Due to the fact that the encoding tree finding by using the structural entropy minimization supports the functional modules of the data space. From the encoding tree, i.e., the decoder, the functional modules of the data space can be interpreted.

\item {\it Tree of abstractions}

By extracting the remarkable common features of the functional modules determined by the decoder, we are able to construct a tree of abstractions. The tree of abstractions defined so, provides the basis for intuitive reasoning.

\item {\it Intuitive reasoning}

When new data points are observed, we may use the tree of abstractions to encode the new data points into the encoding tree to construct an updated encoding tree or decoder.

By using the tree of abstractions, the updating of encoding tree or decoder has become a dynamical and local algorithm. 

\end{enumerate}

These mechanisms make our machine SiLeM essentially different from the currently existing machine learning algorithms. Our machines learn from observing, by associating, computing, abstracting, and intuitive reasoning. The criterion of our machines is the semantical interpretation of the decoder, that is, the encoding tree found from the syntax structure of a data space by an information optimization procedure. These features make our learning machines highly similar to human learning.

\subsection{The principles of SiLeM}

The principles of the structural information learning machinery (SiLeM) include:

\begin{enumerate}
\item []{\bf Principle 1}: {\it The combination of syntax and semantics}.

Real world objects certainly consist of a syntax, a semantics and noises. The goal of learning is to discover the laws or rules of real world objects. Therefore, learning must deal with syntax, semantics and noises. Noises are part of uncertainty. Decoding information maximization excludes the perturbation by noises. Laws or rules themselves consist of a syntax that supports the semantics of the real world objects. Semantics is the knowledge of the real world. Semantical interpretation of a syntactical structure of real world objects is the criterion for both human learning and our structural information learning machines.

\item []{\bf Principle 2}: {\it The merging of computation and information}.

The merging of computation and information is the foundation for our learning model.

\item []{\bf Principle 3}: {\it Tree of abstractions}.

Abstracting is perhaps one of the fundamental differences between human and computer. The decoder of our learning machines is an encoding tree, which naturally supports both a knowledge tree and a tree of abstractions. 

Trees of abstractions empower human to understand hugely complex systems such as countries, societies etc. This remarkable character of our model would empower structural information learning machines learn the universe from small datasets to huge massive datasets.

\item []{\bf Principle 4}: {\it The combination of locally logical reasoning and globally intuitive reasoning}.

Intuitive reasoning is perhaps one of the most important abilities of human. Our structural information learning machines explicitly realize the function of intuitive reasoning. More importantly, intuitive reasoning is realized by our learning machines in a way of principle-based, that is, to gain the maximum amount of information.

\item []{\bf Principle 5}: {\it The merging of encoding, decoding and optimizing}.

This realizes the merging of information and computation.

\end{enumerate}

Principle 1 provides us the criterion for our learning machine that the syntax decoder must have a semantical interpretation, and  that the semantical modules must have a supporting structure. This naturally solves the interpretability problem of the learning machines. 
In human learning, people learn by using brain, eyes, hands, ears etc simultaneously. Especially, people learn not only the syntax, but also semantics of an object simultaneously. In addition, semantics plays a crucial role in human reasoning. In most cases, people get the semantics first, then find the formal proof. This means that semantics and syntax play role simultaneously in human reasoning. Our learning machines SiLeM distinct semantics and syntax, explore the roles of both syntax and semantics in learning.
Principle 1 very well captures the nature of unifying of both semantics and syntax together in human learning.

Both information and computation are fundamental to the current computer science and artificial intelligence. As mentioned before, structural information theory is a new theory of the merging of information and computation. The key to our SiLeM is exactly the structural information theory. In this sense, SiLeM is a learning machine built based on the merging of the two fundamental concepts of information and computation. Principle 2 reflects this new character of the Structural information learning machinery.

Principle 3 is represented by the encoding tree and knowledge tree. In human learning, it is obvious that there are trees of abstracting. Different levels of abstractions correspond to a hierarchy of concepts. Our SiLem naturally realizes the hierarchy of abstracting. 

Principle 4 reflects the character of the combination of locally logical reasoning and globally intuitive reasoning of human learning. To understand this, let us check the differences between logical reasoning and intuitive reasoning. In this paper, we interpret ``intuitive reasoning" as the reasoning by using the knowledge and laws a learner has already built. According to this understanding, the structural information learning machines realize this mechanism when new data points are observed. Logical reasoning is a type of computation, and computation is a type of optimization. Since the structural information learning machines optimize the amount of information gained from either constructing a data space or encoding or decoding of a data space. Therefore the structural information learning machines certainly perform logical reasoning. From the point of view of Turing machines, computation is a local operation, in the sense that, during the procedure of a computation, at any time step, Turing machines see only a local area of the configuration of the computation, that is, a few cells of the working tape, one state of the machine, and a few symbols of an alphabet. By this reason, we interpret computation as a local operation, and hence logical reasoning is a local operation. Apparently, the combination of both local reasoning and intuitive reasoning is a remarkable character of our structural information learning machinery. This new character makes our SiLeM completely a new model of learning. It is this character, we know that computation and learning are completely different scientific concepts, although learning can be realized by algorithms.

Principle 5 is a validation of the merging of information and computation for the structural information learning machines. It is interesting to notice that although encoding and decoding are both algorithms in practical applications, the concepts of encoding and decoding are core ideas of information theory. SiLeM demands the combination of encoding, decoding and optimizing, which naturally realizes the merging of information theory and computation theory. This character implies that SiLeM is not only completely new, but also is coming from the merging of two fundamental concepts of information and computation.

\subsection{The goal of SiLeM}

The goal of the structural information learning machinery (SiLeM) is:

\begin{enumerate}

 \item To acquire knowledge of real world, and 
 \item To discover the laws of real world.

\end{enumerate}

According to the model of structural information learning machinery, the goal of learning is achieved by gaining information, or equivalently, by eliminating the uncertainty (or entropy) embedded in a system of observed data. In another word, we acquire knowledge and discover laws of nature by gaining information, i.e., by eliminating uncertainty of a system of observed dataset. This is exactly the mechanism of human learning.

The essence of the information optimization is to distinguish the laws from noises in a complex system of observed dataset. The theoretical limitation of the structural information learning machinery (SiLeM) is hence to discover the laws of nature, which is important, but is insufficient for us to understand the concept of intelligence. Of course, SiLeM is an information theoretical model of learning. Theoretically speaking, SiLeM is able to discover the laws of nature, provided that it eliminates all the uncertainty embedded in a system. 

However, even if SiLeM realizes the theoretical goal of discovering the laws of nature, it is still insufficient to fully capture the essence of intelligence. The reason is that, in the generation of intelligence, the concept of {\it creating} or {\it designing} must play an essential role. In fact, humans {\it create} or {\it construct} many things based on the knowledge and laws they learnt. Human intelligence consists of both discovering the laws of nature and creating things based on laws. This gives rise to a well-defined description of human intelligence. To better understand the concept of intelligence, we need a model to investigate the concept of creating or designing. For this, we need to find the motivation and mechanism for the action of creating and designing. The author of the present article believes that game is the motivation and mechanism of creating and designing. This calls for a new theory of game, {\it structural game theory}, the author proposed.

Figure 6 depicts the framework of the structural information learning machinery.

\begin{figure}
\centering

    \includegraphics[width=0.8\textwidth]{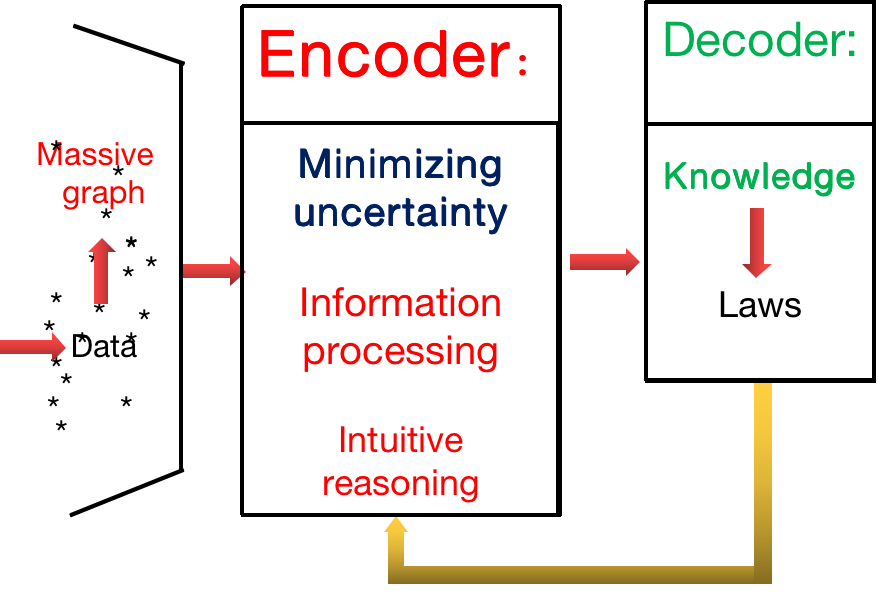}
    \caption{Structural Information Learning Machinery}
    %\label{fig:Learning Model}
\end{figure}

Figure 6 intuitively represents the procedure of a structural information learning machine. It proceeds as follows:

\begin{enumerate}
\item Observing data
\item Linking data to generate a graph, i.e., the data space, by the principle of maximizing the decoding information
\item Encoding the graph by structural entropy minimization principle to find the decoder of the data space
\item Interpreting the semantics from the found encoding tree to generate a knowledge tree
\item Abstracting from the knowledge tree and decoder to discover the laws of the data space
\item When new data is observed, the machine encodes or decodes by using the tree of abstractions. This is a step of intuitive reasoning.

\end{enumerate}

%\subsection{One more useful character of SiLeM: Learning from small dataset}

\section{Encoding Tree Method: Algorithms for Structural Information Learning Machines}\label{sec:encoding-tree-method}

A structural information learning machine $L$ consists of five phases:

\begin{enumerate}
\item [(1)] (Linking data) We construct a data space for a dataset by maximizing the {\it decoding information}, i.e., the decodable information from the construction of the data space. 
\item [(2)] (Decoder) We find the decoder by an encoding tree that minimizes the structural entropy. A decoder is hence the syntactically data structure which ensures that using the data structure, the uncertainty embedded in the original data space is minimized.
\item [(3)] (Knowledge Tree) We build a knowledge tree by  interpreting the semantics of the decoder. 
\item [(4)] (Abstracting) We extract the rules or laws from both the decoder and the knowledge tree of the data space to construct trees of abstractions.
\item [(5)] (Intuitive reasoning) We use the rules or laws from the trees of abstractions obtained from abstracting in future encoding and decoding, when new data points are observed.
\end{enumerate}

We have introduced algorithms for the phases (3), (4) and (5) above. 

For phase (1) and phase (2) above, the key is to find an encoding tree $T$ for a graph $G$, say. 

For phase (1), the optimization problem is to find an encoding tree $T$ such that

\begin{equation}\label{eqn:C}
\max\{\mathcal{C}^T(G)\},
\end{equation}
where $\mathcal{C}^T(G)$ is given in Equation (\ref{eqn:compressing-tree}).

For phase (2), the optimization problem is to find an encoding tree $T$ such that 

\begin{equation}\label{eqn:D}
\min\{\mathcal{H}^T(G)\},
\end{equation}
where $\mathcal{H}^T(G)$ is given in Equation (\ref{eqn:H-T-O}).

For both the optimization problems in Equations (\ref{eqn:C}) and (\ref{eqn:D}), we may design the algorithms by various kinds of tree operations. Due to this reason, we call it the {\it encoding tree method}. The encoding tree method is a general framework, instead of a specific strategy of algorithms, for finding an optimal encoding tree.

The encoding tree method proceeds as follows:

\begin{enumerate}
\item We start with an initial encoding tree $T_0$ of the following form.

The root is $\lambda$, and the root has $n$ many immediate successors, each of which is associated with a singleton $\{v\}$ for some vertex $v$ in $G$.

Suppose that $T$ is the currently constructed encoding tree.

 For two nodes $\alpha,\beta\in T$, we may define different operations by using the branches with roots $\alpha$ and $\beta$, respectively. For example, the merging operations and combining operations in \cite{LP2016a, L+2018}.
 
 \item Choose $\alpha$ and $\beta$ such that the operation using $\alpha$ and $\beta$ gains the maximal increment of information for phase (1) or gets the maximal entropy elimination for phase (2).
 
 \item Let $T^{\alpha,\beta}$ be the encoding tree obtained from $T$ by executing the operations using $\alpha$ and $\beta$.
 \item Set 
 
 - $T\leftarrow T^{\alpha,\beta}$.
 
 \item Go back to step 2 above.

\end{enumerate}

The encoding method is a general framework for information optimization problems corresponding to both phase (1) and phase (2) optimization problems of the structural information learning machinery. The encoding tree method leaves a huge room for information optimization.

Although there is no theoretical guarantee whether the algorithms find the optimum or almost optimum encoding tree, there are simple and efficient greedy algorithms for approximating an optimal encoding tree. It is interesting to notice that, even if simple greedy algorithms have remarkable performance in applications, in both efficiency and quality, see, for instance \cite{LP2016a,L+2018}.

The encoding tree method above has a number of advantages, such as:

\begin{enumerate}
\item [1)] Computation is local

Due to the locality of the operation using two tree nodes $\alpha$ and $\beta$, the computation of information gained or entropy eliminated from encoding tree $T$ to encoding tree $T^{\alpha,\beta}$ depends only on the branches with root $\alpha$ and $\beta$ in $T$. In addition, since the definition of both $\mathcal{C}^T(G)$ and $\mathcal{H}^T(G)$ has the additivity. For a given pair $\alpha$ and $\beta$, the corresponding incremental $\Delta_{\alpha,\beta}$ depends on only a small number of terms, and hence is locally computable.

\item [2)] The number of pairs $(\alpha,\beta)$ is restricted. For instance, we consider only the pairs $(\alpha,\beta)$ such that there are edges between $T_{\alpha}$ and $T_{\beta}$. Since, otherwise, the operations maybe proved inappropriate. In fact, we are able to prove some results showing that, there are only a small number of pairs $(\alpha,\beta)$ of tree nodes $\alpha$ and $\beta$ such that the operations 
with them require attention, all other pairs can be simply ignored. This further speeds up the algorithms for finding the desired encoding tree.

[An important research topic of the encoding tree method is to prove some lemmas to determine some conditions for the pairs $(\alpha,\beta)$ of tree nodes $\alpha$ and $\beta$ that require attention.]

\item [3)] For different types of applications, we may consider only the restricted type of encoding trees. For example, for most real world applications, the heights of the encoding trees are $2$ or $3$, which roughly correspond to the three-dimensional shape of objects in the real world. (In \cite{L+2018}, the algorithm deDoC was developed. It is the first principle-based (no any hand-made parameter) method to find topologically associating domains for genomes. The method deDoC is remarkably better than the currently existing methods.)

In applications of some other areas such as in natural language understanding, 
it would be better to restrain the height of encoding trees to be some number, $5$ say, which reflects the levels of abstractions of human learning.

[This provides the rooms of choices needed for applications in different areas. Of course, the motivation of restricting of the type of encoding trees is mainly for efficiency of algorithms.]

\item [4)] An important property of the encoding tree method is that, a local operation realizes a global benefit.

For example, for the structural entropy minimization problem. Given an encoding tree $T$, we choose $\alpha,\beta$ such that

\begin{equation}
\Delta_{\alpha,\beta}=\mathcal{H}^T(G)-\mathcal{H}^{T^{\alpha,\beta}}(G)
\end{equation}
is maximized, where $T^{\alpha,\beta}$ is the tree obtained from $T$ by operation at $\alpha$ and $\beta$.

We notice that both $T$ and $T^{\alpha,\beta}$ are global quantities. Hence, $\Delta_{\alpha,\beta}$ is actually a global quantity. However, $\Delta_{\alpha,\beta}$ has a local representation. Because, most terms in $\mathcal{H}^T(G)$ and $\mathcal{T^{\alpha,\beta}}(G)$ cancelled each other.
This feature ensures that when we choose $\alpha$ and $\beta$, the operation at $\alpha$ and $\beta$ has realized a global benefit.

Therefore, we realize that {\it a local action} gets {\it a global benefit}.

\item [5)] Another important feature of the encoding tree method is {\it robustness}, in the sense that, the algorithms work almost equally well even if one imputes some noises to the data space. To understand this, we consider the following case. Given a graph $G$, we may find an encoding tree $T$. Suppose that $G'$ is a graph obtained from $G$ by making a small number of changes. Then encoding tree $T$ of $G$ can be easily updated to be an encoding tree $T'$ of $G'$ such that $T'$ is different with $T$ only at some local areas. This feature can be easily observed from the encoding tree method.

\item [6)] The fundamental property of the encoding tree method is the principle of the maximization of decoding information from the encoding. This is an information theoretical principle, and of course, a general principle for our learning machinery. This means that although learning different real world objects has different semantics, there is a universal syntactical principle for the different learnings. This universal principle is to maximize the decoding information from the observed datasets. This principle can be realized because: (1) gaining information is mathematically equivalent to eliminating uncertainty, and (2) according to structural information theory, eliminating uncertainty can be transformed into an information optimizing problem.

\item [7)] According to the definition of the structural information learning machinery (SiLeM), there is no any hand-made parameter choice 
in any step of the learning machines. This is remarkably different from the currently known learning algorithms.

\item [8)] Encoding tree method provided by the structural information learning machinery (SiLeM) would be a general method for information optimization, and for learning and intelligence algorithms.

\item [9)] The most fundamental idea of the structural information learning machinery (SiLeM) is perhaps that the model better explores and captures the mathematical essence of the concept ``learning". This provides a foundation for us to mathematically study the relationships between learning and computing, between learning and intelligence, and between information and intelligence etc.

\end{enumerate}

\section{Theoretical Limitations of the Structural Information Learning Machinery (SiLeM)}\label{sec:limits-of-SiLeM}

Our model SiLeM is a general model of learning based on observation. However, as usual, any mathematical model has limitations. Our structural information learning machinery has its own limitations. Our SiLeM model shows that everything is learnable, provided that dataset can be observed, that many things, especially, non-mathematical objects are hard to be exactly learnt, in the sense that, every bit of information embedded in the object is decoded, and that learning is an approximating of information from a real object, instead of exactly computing the real world object.
This means that, for every real world object, we may always eliminate certain amount of uncertainty of the object, but we may never eliminate all the uncertainty embedded in the object.
Therefore, our structural information learning machinery (SiLeM) is a universal, but never omnipotent, model of learning.

In particular, there are the following limitations of the machines:

\begin{enumerate}
\item The machines learn from observing, and only from observing.

This means that our machines do not say anything about the objects the machines have not observed. On one hand, if the machines observe something of an object, we can gain certain amount information of the object, but on the other hand, if the machines do not observe anything about an object, we can gain only zero information from the object.
Clearly, observed objects form only a small set of the universe.

Our structural information learning machinery (SiLeM) differs from all the existing learning models that are built to try to learn every mathematical function from a hypothesis space, assuming that real world object is represented by a function in the hypothesis space.
Our model does not have such as a hypothesis. We don't assume any function represents a real world object, and we don't assume a certain space contains the function desired. All we assumed is that laws of an object is embedded in the space of the data points observed from the object. The mission of our learning machine is to find the laws from the noisy data space from the observations of the object. 

Our theory also shows that any statement about an object one never observes is groundless and is hence false.

\item Observation is key to our learning model.

As mentioned in 1 above, according to our model, without observation, then without any information learnt. Therefore, observing is the first fundamental mechanism of learning. To learn from the world, we need to observe the world first. Observation is probably incomplete, or insufficient, or even incorrect. This is, of course, a limitation of our structural information learning machinery (SiLeM). This also suggests that understanding human vision is an important direction for learning and for intelligence, which would be the first step for us to build human-like machines in the future.

\item Usually, real world objects are hard to be fully, exactly and completely observed at once. Observation of real world objects could be a cumulative procedure over the time.

\item An observation plays different roles in learning when it is associated with different knowledges.

It is also possible due to the differences of knowledge cumulation, the same observation of an object gives rise to different associations of the observation to the knowledges, and hence plays different roles in the learning procedure. Therefore, our model of SiLeM depends on the knowledge and laws it has already built. This is also a limitation of our SiLeM machines. However, it is reasonably true that people with different backgrounds may have different understandings on the same event observed. Our structural information learning machinery (SiLeM) has the same character.

\item There are systems from which little information can be decoded. 

For example, there are systems satisfying:

\begin{enumerate}
\item there is rich information embedded in the system, and
\item there is no a decoder that significantly eliminates the uncertainty embedded in the system.
\end{enumerate}

Examples of such systems include such as complete graphs or the graphs with a huge number of edges, see Proposition \ref{pro:complete}. 
However, systems observed from the real world are unlikely to be such cases. Physical systems observed from the real world are usually sparse, allowing an encoding tree to significantly decode the information embedded in the systems.

\item Our model does not assume a unique correct answer for learning. Our criterion is the semantical interpretability of a decoder, which is a syntactical structure, found based on the principle of maximization of decoding information (or minimization of structural entropy). This naturally solves the interpretability problem of learning. 

However, the principle above has no a mathematical proof. All we have is a high-level hypothesis, or thesis that a decoder (an encoding tree) has already significantly eliminated syntactical uncertainty of a system, certainly supports a structure of the functional modules (semantics) of the system. For this reason, our model resolves the interpretability problem of learning by using a thesis, instead of a mathematical proof.

Of course, due to the fact that, decoders may not be unique, it is possible for a system to have different semantics. This is also reasonable, similar to the case in model theory, in which an axiomatic mathematical system may have different models, or similar to the case in natural language understanding that a word may have different semantical interpretations.

\item As mentioned before, it is perhaps the major disadvantage of our structural information learning machinery (SiLeM) is that there is no mechanism for {\it creating} or {\it designing} in the model. All the goals of the structural information learning machinery is to gain information from the dataset a machine observed by eliminating uncertainty occurred in the dataset. The machines have no desire to create, and have no principle to decide what it can and will create.

\item The disadvantage in item 7 above indicates that there is a theoretical gap between learning and intelligence.

\end{enumerate}

\section{Conclusions and Discussion}\label{sec:con}

We have proposed a new model of learning, the structural information learning machinery, written SiLeM. A SiLeM machine learns the knowledge and laws of nature by observing the real world. The high-level of the SiLeM machines is the hypothesis that real world 
consists of both laws and noises, and that a real world system is a structure in which laws are embedded in a system of massive noises.
According to this hypothesis, the goal of {\it information processing} is to distinguish the laws from noises in a real world system. Fortunately, the structural information theory \cite{LP2016a,L+2018} provided a mathematical theory supporting such a mission. 
The model SiLeM is built based on the structural information theory. 

The contributions of the structural information learning machines (SiLeM) include:

\begin{enumerate}
\item The SiLeM machines show that learning is to gain information, that to gain information is to eliminate uncertainty, and that to eliminate uncertainty can be reduced to an information optimization problem.

This explores and captures the mathematical essence of learning.

\item The SiLeM learns the laws of nature by observing.

This feature makes the SiLeM like very much human learning.

\item The SiLeM learns by linking, connecting and associating data to data.

So our SiLeM accepts the hypothesis of connectionism. As a matter of fact, SiLeM moves forward further. According to Shannon's theory, we know that entropy is the amount of uncertainty, embedded in a probability distribution or a random variable, and information is the amount of uncertainty that has been eliminated. Shannon's theory characterizes the entropy of a random variable, and shows that information occurs in and only in transformation from one point to another. Let us summarize the Shannon theory as follows: Information is gained only in communication. 

Unfortunately, Shannon's theory only measures the entropy of a random variable or probability distribution. Clearly, uncertainty exists in not only random variables, but also in any complex systems. For the latter, Shannon did not say anything. Our structural entropy measures the amount of uncertainty of a complex system. The metric leads to fundamental theory of the information embedded in physical systems \cite{LP2016a}.

According to our structural information theory, we know that information exists in physical systems, precisely, in the interactions or communications in systems. This means that without interactions or communications in a system, there will be no information gain.

Our theory assumes that any procedure of learning is to gain information, or equivalently, to eliminate the uncertainty embedded in a system. 
This implies that without a system, there will be no learning.

More importantly, we believe that information is certainly one of the key factors of {\it artificial intelligence} (AI). Therefore, the SiLeM provides new insights for us to understand AI. 

The new insight for AI got from the structural information learning machines is that interactions and communications in complex systems are the foundations for intelligence generation. Therefore, a possible mathematical definition of artificial intelligence should be some well-defined metrics capturing the ultimate states of complex systems that interacting and communicating, for which the mechanism is {\it creation based on laws}.

\item The SiLeM learns both syntax and semantics of a data space.

\item The SiLeM can abstract rules and laws from the decoder and knowledge tree to construct trees of abstractions, in which
the concept of knowledge tree could be a general model for knowledge representation, and trees of abstractions would be the core idea of intuitive reasoning.

\item The SiLeM performs not only computation, consisting of local actions, but also performs intuitive reasoning, consisting of global actions. Here we understand the reasoning by using knowledge, laws and trees of abstractions as intuitive reasoning.

\item The SiLeM not only provides new insights for us to understand the mathematical essence of learning and intelligence generation, but also provides new approaches to important new applications such as big data analysis, biological and medical data analysis, and natural language understanding.
\item The SiLeM may provide principles for the hand-made parameter choices in many currently existing learning algorithms. This provides new insights for us to interpret and understand the current learning algorithms.

\item Mathematically speaking, our structural information learning machines show that learning is different from computation. Computation is a mathematical concept, dealing with computable functions and computing devices. Learning is largely a structural information theoretical concept. A learning procedure is to gain information from the datasets observed from real world objects for which the goal is to build knowledge and to discover the laws of the world. Furthermore, the universe of computation is mathematics, the operations of computation are local, the goal of computation is to efficiently compute the mathematically defined object functions. However, the universe of learning is the real world, learning deals with syntax, semantics and noises, the operations of learning are both logical reasoning and intuitive reasoning, the goal is learning is to discover the laws of nature by observing.

\end{enumerate}

The structural information learning machines provide new ideas for us to understand the concept of learning and intelligence. Equally important, the new machines are promising for us to develop new applications in a wide range of artificial intelligence. Theoretically speaking, the new machines provide a wide range of new applications, to name a few below, for example:

\begin{itemize}
\item Natural language understanding and natural language processing

From the point of views of information theory and structural information theory, information exists in communications. This means that if there is no communication, then there is no information gain. We may assume a basic hypothesis that information is key to intelligence and that human beings are highly intelligent. Where does the human intelligence come from? The most obvious phenomenon is that human beings communicate from the first day of birth. From this, we may assume that without communications, it is hard to have intelligence even if for human beings. This is probably true. If a baby was sent to an isolated island where no communications at all, then the baby may not grow up as a person at the level of intelligence of people in a normal society. Human beings communicate through natural languages. The arguments above show that natural language understanding and processing are key to the generation of human intelligence. For this reason, natural language understanding and processing could be crucial for us to capture the essence of human intelligence and then artificial intelligence. 

Natural language understanding naturally involves both syntax and semantics, both logical reasoning and intuitive reasoning, knowledge trees and trees of abstractions, and involves complex systems of languages and knowledges. The structural information learning machines provide all these ingredients for natural language understanding, and provide the mechanisms for learning from observations. 

The arguments above imply that well-defined structural information learning machines have the potential to realize natural language understanding and natural language processing.

\item Biological and medical data analysis

Biology and medical sciences provide a rich sources for data analysis. Analysis of biological and medical datasets should obey the laws of life science. There are plenty of life science laws, of course. However, the fundamental law of life science could be the natural selection, consisting of heredity and variation. Mathematically, heredity and variation intuitively correspond to copy and randomness, respectively. 
Then natural selection is a procedure of growing up by laws in a noisy environment. The result of a natural selection is just a system in which laws embedded in a noisy structure. The structural information theory provides for the first time a principle to distinguish laws from noises in a complex system. This is to say, it seems that the structural information theory naturally obey the laws of natural selection. Therefore, the structural information learning machines may work well on data analysis for biology and medical sciences.

\item Robots that learn

In many cases, robots are required to autonomously plan and act. This requirement can be realized by a machine that learns from observations. This feature is very well captured by our structural information learning machines. This high-level analogy between the requirement of robots and the mechanisms of the structural information learning machines indicates that well-built SiLeM may realize the goals of robots.

\end{itemize}

The structural information learning machinery implies that mathematical understanding of some grand challenges in the area of artificial intelligence is possible. A few examples include:

\begin{enumerate}

\item [(1)] The relationship between learning and intelligence

\item [(2)] The relationship between information and intelligence

\end{enumerate}

For the two challenges above, we have already had the mathematical definitions for learning and information, both of which must be the key ingredients of intelligence. From the point of view of 21st century science, we need a mathematical understanding of artificial intelligence.


\begin{thebibliography}{99}

\bibitem{AM2001}
Adler, M. and Mitzenmacher, M.
Towards compressing web graphs. In Proc. of the IEEE Data Compression Conference, pp: 203-212, 2001.

\bibitem{B2003}
Brooks, F. P.
Three great challenges for the half-century-old computer science.
J. ACM, 50 (1), 25-26, 2003.

\bibitem{CS2009}
Choi, Y. and Szpankowski, W.
Compression of graphical structures, Proc. ITTT International Symposium on Information Theory, 364-368, 2009.

\bibitem{CL2006}
F. Chung and L. Lu.
\newblock Complex graphs and networks.
\newblock {\em American Mathematical Society}, ISBN-13: 978-0-8218-3657-6, (2006).


\bibitem{EGP1966}
Erd\"os, P., Goodman, A. and P\'osa, L.
\newblock On the minimal number of vertices representing the edges of a graph.
\newblock{\em Canad. J. Math.}, {\bf 18}, pp 106-112, 1966.



\bibitem{H1952}
Huffman, D. A.
A method for the construction of minimum redundancy codes.
Proc. Inst. Rail. Engin. 40, 1098-1011, 1952.


\bibitem{LBH2015}
LeCun, Y., Bengio, Y. and Hinton, G.
\newblock Deep learning.
\newblock{\em Nature}, 436, 521, 2015.


\bibitem{LGT2014}
Lee, J., Gharan, S., $\&$ Trevisan, L.
\newblock Multi-way spectral partitioning and higher-order Cheeger inequalities.
\newblock{\em Journal of the ACM}, 61(6), 37:1 - 37:30 (2014).



\bibitem{LHLP2016}
Li, A., Hu, Q., Liu, J. and Pan, Y.
\newblock Resistance and security index of networks: Structural information perspective of network security.
\newblock{\em Scientific Reports}, {\bf 6}: 26810, pp 1-24, 2016.


\bibitem{LLPZ2015}
Li, A., Li, X., Pan, Y. and Zhang, W. 
\newblock Strategies for network security.
\newblock{\em Science China, Information Sciences}, Jan. 2015, Vol. 58 012107:1-012107:14.


\bibitem{LP2016}
Li, A. $\&$ Pan, Y.
\newblock A theory of network security: Principles of natural selection and combinatorics.
\newblock{\em Internet Mathematics}, Vol. 12, pp 145-204, 2016.


\bibitem{LP2016a}
Li, A. $\&$ Pan, Y.
\newblock Structural information and dynamical complexity of networks.
\newblock{\em IEEE Transactions on Information Theory}, Vol. 62, No. 6, pp 3290-3339, 2016.

\bibitem{LYP2016}
Li, A., Yin, X. $\&$ Pan, Y.
\newblock Three-dimensional gene map of cancer cell types: Structural entropy minimisation principle for defining tumour subtypes.
\newblock{\em Scientific Reports}, 6:20412, DOI:10.1038, srep20412, 2016.

\bibitem{L+2018}
Li, A., Yin, X., Xu, B., Wang, D., Han, J., Wei, Y., Deng, Y., Xiong, Y. $\&$ Zhang, Z.
\newblock Decoding topologically associating domains with ultra-low resolution Hi-C data by graph structural entropy.
\newblock{\em Nature Communications}, 9: 3265 DOI: 10.1038/s41467-018-05691-7, 2018.

\bibitem{LZP2016}
Li, A., Zhang, X. $\&$ Pan, Y.
\newblock Resistance maximization principle for defending networks against virus attack.
\newblock{\em Physica A}, 466, 211 - 223 (2016).


\bibitem{Liu+2019}
Liu, Y., Liu, J., Zhang, Z., Zhu, L. and Li, A.
\newblock From structural entropy to community deception.
\newblock{\em NeurIPS}, Vancouver, Canada, 2019.

\bibitem{LNP1980}
Lov\'asz, L., Nesetril, J. and Pulte, A.
\newblock On the product dimension of graphs.
\newblock{\em J. Comb. Th. (B)}, {\bf 28}, pp 47-67, 1980.

\bibitem{N1990}
Naor, M.
Succinct representation of general unlabelled graphs. Discrete Applied Mathematics, 28(3), 303-307, 1990.



\bibitem{P2007}
Peshkin, L.
Structure induction by lossless graph compression. In Proc. of the IEEE Data Compression Conference, 53-62, 2007.

\bibitem{S2004}
Savari, S. A.
Compression of words over a partially commutative alphabet, IEEE Trans. on Information Theory, 50, 1425-1441, 2004.


\bibitem{S1948}
Shannon, C.
A mathematical theory of communication.
Bell Syst. Tech. J., 27 (3), 379-423, 27 (4), 623-656, 1948.


\bibitem{S1953}
Shannon, C.
The lattice theory of information.
IEEE Trans. Information Theory, 1 (1), 105-107, Feb., 1953.
\bibitem{SBB2008}
Sun, J., Bolt, E. M. and Ben-Avraham, D.
Graph compression-save information by exploiting redundancy, J. of Statistical Mechanics: Theory and Experiments, P06001, 2008.

\bibitem{T1936}
Turing, A. M.
\newblock On computable numbers, with an application to the entscheidungsproblem.
\newblock{\em Proceedings of the London Mathematical Society}, Ser. 2, Vol. 42, pp. 230-265, 1936.




\bibitem{T1984}
Turn, Gy. On the succinct representation of graphs, Discrete Applied Mathematics, 8 (3), 289-294, 1984.




\end{thebibliography}
\end{document}